\documentclass{article}

\newif\ifdraft

\PassOptionsToPackage{numbers, compress, sort}{natbib}

\usepackage[final]{neurips_2022}

\usepackage[utf8]{inputenc} %
\usepackage[T1]{fontenc}    %
\usepackage{hyperref}       %
\usepackage{url}            %
\usepackage{booktabs}       %
\usepackage{amsfonts}       %
\usepackage{nicefrac}       %
\usepackage{microtype}      %
\usepackage{xcolor}         %

\usepackage{amsmath}
\usepackage{amsthm}
\usepackage{listings}
\usepackage{algorithm}
\usepackage{algpseudocode}
\usepackage{subcaption}
\usepackage{caption}
\usepackage{graphicx}
\usepackage{tcolorbox}
\usepackage{wrapfig}
\usepackage{mathabx}
\usepackage{thmtools}
\usepackage{thm-restate}

\usepackage{multibib}
\newcites{supp}{References}

\definecolor{todored}{rgb}{0.7,0,0.05}
\definecolor{notegreen}{rgb}{0,0.5,0.1}

\ifdraft
\newcommand\todo[1]{\textcolor{todored}{[TODO: #1]}}
\newcommand\note[1]{\textcolor{notegreen}{[NOTE: #1]}}
\newcommand\lo[1]{\textcolor{teal}{[LO: #1]}}
\else
\newcommand\todo[1]{}
\newcommand\note[1]{}
\newcommand\lo[1]{}
\fi

\newcommand\change[1]{#1}

\definecolor{mydarkblue}{rgb}{0,0.08,0.45}
\hypersetup{ %
pdftitle={},
pdfauthor={},
pdfkeywords={},
pdfborder=0 0 0,
pdfpagemode=UseNone,
colorlinks=true,
linkcolor=red,
citecolor=blue,
filecolor=mydarkblue,
urlcolor=mydarkblue,
pdfview=FitH}

\definecolor{commentgreen}{RGB}{2,112,10}
\definecolor{eminence}{RGB}{108,48,130}
\definecolor{weborange}{RGB}{255,165,0}
\definecolor{frenchplum}{RGB}{129,20,83}

\definecolor{sdvired}{HTML}{D55E00}
\definecolor{pyrogreen}{HTML}{009E73}
\definecolor{bbviblue}{HTML}{56B4E9}

\usepackage{listings}
\newcommand\KL{\mathrm{KL}}
\newcommand\ELBO{\mathcal{L}}
\newcommand\qind{q}
\newcommand\qslp{q_{k}}
\newcommand\tildeqslp{\tilde{q}_{k}}
\newcommand\branchind{I_{{\rm branch}}}
\newcommand\nonbranchind{I'}
\newcommand\tildexcal{\mathcal{X}'}

\newcommand\colorsdvi{\textcolor{sdvired}{SDVI} }
\newcommand\colorpyro{\textcolor{pyrogreen}{Pyro AutoGuide} }
\newcommand\colorbbvi{\textcolor{bbviblue}{BBVI} }

\title{Rethinking Variational Inference for Probabilistic Programs with Stochastic Support}

\author{%
  Tim Reichelt\textsuperscript{1} \qquad Luke Ong\textsuperscript{1,2} \qquad Tom Rainforth\textsuperscript{1}\\
  \textsuperscript{1} University of Oxford \\
  \textsuperscript{2} Nanyang Technological University, Singapore \\
  \texttt{\{tim.reichelt,lo\}@cs.ox.ac.uk \quad rainforth@stats.ox.ac.uk}
}

\begin{document}

\maketitle

\begin{abstract}
We introduce \emph{Support Decomposition Variational Inference} (SDVI), a new variational inference (VI) approach for probabilistic programs with stochastic support.
Existing approaches to this problem rely on designing a single global variational guide on a variable-by-variable basis, while maintaining the stochastic control flow of the original program.
SDVI instead breaks the program down into sub-programs with static support, before automatically building separate sub-guides for each.
This decomposition significantly aids in the construction of suitable variational families, enabling, in turn, substantial improvements in inference performance.
\end{abstract}

\section{Introduction}

Probabilistic programming systems (PPSs) enable users to express probabilistic models with computer programs and provide tools for inference.
Many PPS, such as Stan \citep{carpenter2017Stan} or PyMC3 \citep{salvatier2016Probabilistic}, limit the expressiveness of their language to ensure that the programs in their language always correspond to models with static support---i.e. the number of variables and their support do not vary between program executions.
In contrast, \emph{universal PPSs} \citep{tolpin2016Design,goodman2008Churcha,bingham2019Pyro,ge2018Turing,cusumano-towner2019Gen,mansinghka2014Venture,narayanan2016Probabilistic,goodman2014WebPPL,murray2018probabilistic} can encode programs where the sequence of variables itself---not just the variable values---changes between executions, leading to models with stochastic support. 
These models have applications in numerous fields, such as natural language processing \citep{manning1999Foundations}, Bayesian Nonparametrics \citep{richardson1997Bayesian}, and statistical phylogenetics \citep{ronquist2020Universal}.
A wide range of simulator-based models
similarly require such stochastic control flow~\citep{baydin2019Etalumis,le2017Inference,gram-hansen2019Hijacking}.

The effectiveness of PPSs is heavily reliant on the underlying inference schemes they support.
Variational inference (VI) is one of the most popular such schemes, both in PPSs and more generally~\citep{blei2017Variational,kucukelbir2015Automatic,agrawal2020advances}.
This popularity is due to its ability to use derivatives to scale to large datasets and high-dimensional models~\citep{zhang2018advances,hoffman2013Stochastica,rezende2015Variational,kingma2014AutoEncoding}, often providing much faster and more scalable inferences compared to Monte Carlo approaches \citep{brooks2011Handbook}.
To provide the required derivatives, a number of modern universal PPSs---such as Pyro \citep{bingham2019Pyro}, ProbTorch~\citep{siddharth2017learning}, PyProb~\citep{baydin2019Etalumis}, Gen~\cite{cusumano-towner2019Gen}, and Turing~\cite{ge2018Turing}---have introduced automatic differentiation \citep{baydin2018automatic} capabilities for programs with stochastic control flow.
One of the core aims behind these developments was to support VI schemes in such settings~\citep{bingham2019Pyro}.

\change{However, constructing appropriate variational families, typically known as guides in PPSs, can be very challenging for problems with stochastic support, even for expert users.  This is because the stochasticity of the control flow induces discontinuities and complex dependency structures that are difficult to remain faithful to and design parameterized approximations for.
Furthermore, while there are a plethora of different automatic guide construction schemes for static support problems \citep{blei2017Variational,kucukelbir2015Automatic,agrawal2020advances}, there is a lack of suitable schemes applicable to models with stochastic support.
Consequently, existing methods tend to give unreliable results in such settings, as we demonstrate in Figure~\ref{fig:motivating_example}.}

We argue that a significant factor of this shortfall is that standard practice---for both manual and automated methods---is to construct the guide on a variable-by-variable basis~\citep{wingate2013Automated,paige2016Automatic,vandemeent2016BlackBox,vandemeent2018Introduction}.
Namely, existing approaches generally use a single global guide that mirrors the control flow of the input program, then introduce a variational approximation for each unique variable.
This is problematic because control flows inherently introduce discontinuities into the program's density function, such that the conditional distribution of each variable will typically change significantly whenever the program path---that is the sequence of random variables---changes.
Thus it is extremely challenging to learn a single approximation for each variable that is appropriate across all paths.
Further, as the set of variables that exist can itself be stochastic, it is difficult for such guides to appropriately condition on previously sampled variables. 
Existing automated approaches, therefore, typically rely on mean-field assumptions~\citep{paige2016Automatic}, thereby forgoing any conditioning on the program path itself, consequently leading to poor approximations for most problems.

To overcome these difficulties, we propose \emph{Support Decomposition Variational Inference} (SDVI), a new VI approach based around a novel way of constructing the variational guide.
\change{SDVI ``rethinks'' the guide construction by breaking it down over \emph{paths}, instead of building it on a variable-by-variable basis.}
Specifically, by utilizing the fact that any program can be reformulated as a mixture of \emph{straight-line programs} (SLPs)~\citep{chaganty2013Efficiently,sankaranarayanan2013Static,zhou2020Divide}---each defined by a unique realization of the path---SDVI constructs the guide as a mixture of sub-guides with static support.
We show that optimizing the variational objective with this guide structure leads to a natural decomposition of the overall
optimization problem into independent sub-problems, each taking the form of a VI with static support.
The sub-guides can thus be effectively constructed and trained using more standard VI techniques, before being recombined to form our overall variational approximation.
To make SDVI accessible to a wide audience, we have implemented it in Pyro \citep{bingham2019Pyro}. 
We evaluate it on a set of example problems with synthetic and real-world data, finding that it provides substantial improvements over existing techniques.

\begin{figure}[t]
  \begin{subfigure}{0.51\textwidth}
    \begin{lstlisting}[numbers=none, basicstyle=\footnotesize\ttfamily]
def model():
  x = sample("x", Normal(0, 1))
  if x < 0:
    z = sample("z1", Normal(-3, 1))
  else:
    z = sample("z2", Normal(3, 1))
  sample("y", Normal(z, 2), obs=2.0)
    \end{lstlisting}
  \end{subfigure}
  \hfill
  \begin{subfigure}{0.4\textwidth}
  \centering
  \includegraphics[width=\textwidth, height=100pt, keepaspectratio]{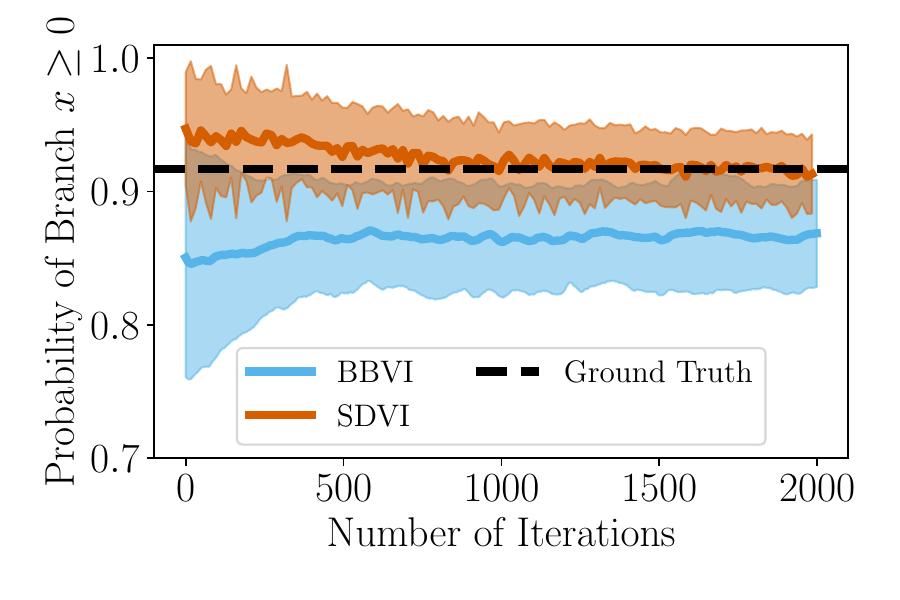}
  \label{fig:x_posterior}
  \end{subfigure}
  \vspace{-11pt}
  \caption{Pyro program with stochastic control flow [Left]. 
  \change{Existing procedures for automatically constructing the guide mirror the control flow of the input program [BBVI, \textcolor{bbviblue}{Blue}]. However, this produces an inherently limited variational family, leading to unsatisfactory performance despite the problem's simplicity.}
  By %
  breaking down the guide over paths, SDVI [\textcolor{sdvired}{Red}] is able to provide accurate inference.
  Results computed over $10^2$ replications, plotted %
  are mean and standard deviation.
  }
  \label{fig:motivating_example}
  \vspace{-14pt}
\end{figure}

\vspace{-7pt}
\section{Background}
\vspace{-4pt}
\subsection{Probabilistic Programs in Universal PPSs}
\label{sec:ppl_intro}
\vspace{-4pt}

PPSs allow users to express probabilistic models and condition on observed data \citep{gordon2014Probabilistic,vandemeent2018Introduction}.
A common mechanism to achieve this is to extend standard programming languages with two new primitives: \lstinline{sample} and \lstinline{observe}.\footnote{
Pyro instead overloads the \lstinline{sample} primitive: %
\lstinline{sample(id,dist,obs=data)}\,$\equiv$\,\lstinline{observe(id,data,dist)}.}
\lstinline{sample(id, dist)} draws samples from the distribution object \lstinline{dist}, where \lstinline{id} is a unique lexical identifier.
\lstinline{observe(id, data, dist)} enables conditioning on an observed outcome \lstinline{data}, where \lstinline{dist} and \lstinline{id} are as before.
For problems admitting a Bayesian formulation, the \lstinline{sample} and \lstinline{observe} terms can informally be thought of as prior and likelihood factors respectively.

\emph{Universal} PPSs %
allow users to write complex models whose support can vary from one execution to the next, e.g.~stochastic branching can mean certain variables only sometimes exist.
This can substantially complicate the process of performing inference.

A probabilistic program in a universal PPS defines an unnormalized \emph{density function} $\gamma(x_{1:n_x})$ over the \emph{raw random draws} $x_{1:n_x} \in \mathcal{X}$---defined as the (sequences of) direct outputs of \lstinline{sample} statements---where $n_x \in \mathbb{N}^+$ is itself potentially random. 
Though each outcome of $x_{1:n_x}$ uniquely defines a program execution, it is notationally convenient to further associate an \emph{address} $a_i$ to each draw $x_i$ that indicates the position in the program the draw was made.
This address can be uniquely defined as the tuple formed by the \lstinline{id} of the \lstinline{sample} and the number of times that \lstinline{sample} has previously been called.
For a given execution of the program, the addresses now form an \emph{address path} $A = a_{1:n_x}$. 

Each \lstinline{sample} statement encountered during execution contributes the factor $f_{a_i}(x_i \mid \eta_i)$ to the program density, where $a_i$ is the address of the \lstinline{sample} statement, $f_{a_i}$ is a parameterized density function, and $\eta_i$ are its associated parameters. 
Similarly, each encountered \lstinline{observe} statement contributes the factor $g_{b_j}(y_j \mid \phi_j)$, with $b_j$ denoting an address, $y_j$ the observed value, $g_{b_j}$ a parameterized density function, and $\phi_j$ its parameters.
Following \citep[§4.3.2]{rainforth2017Automating}, we write the program density function as
\begin{equation}
  \label{eq:ppl_density}
  \gamma(x_{1:n_x}) := \prod\nolimits_{i=1}^{n_x} f_{a_i}(x_i \mid \eta_i) \prod\nolimits_{j=1}^{n_y} g_{b_j}(y_j \mid \phi_j).
\end{equation}
All of $n_x, n_y, a_{1:n_x}$, $\eta_{1:n_x}$, $y_{1:n_y}$, $b_{1:n_y}$, and $\phi_{1:n_y}$ are potentially random variables.
The goal of \emph{inference} is to approximate the conditional distribution of the program, which has normalized density $\pi(x_{1:n}) = \gamma(x_{1:n}) / Z$ with marginal likelihood $Z = \int_{\mathcal{X}} \gamma(x_{1:n_x}) dx_{1:n_x}$ %
and the integral is computed with respect to a reference measure that is implicitly defined by the \lstinline{sample} statements in the program.

\subsection{Variational Inference}
\label{sec:vi_background}

Variational Inference (VI) \citep{wainwright2008graphical,blei2017Variational} solves the inference problem by transforming it to an optimization problem.
Specifically, given an unnormalized joint distribution $\gamma(x)$ and a parameterized distribution $q(x; \phi)$, VI computes the variational parameters $\phi$ such that $q(x; \phi)$ most closely approximates $\pi(x) = \gamma(x) / Z$.  %
This is most commonly done by maximizing
the \emph{Evidence Lower Bound} (ELBO)
     $\ELBO(\phi) := \mathbb{E}_{q(x; \phi)}\left[\log \nicefrac{\gamma(x)}{q(x; \phi)}\right]$
via stochastic gradient ascent using Monte Carlo estimates of $\nabla_{\phi} \ELBO(\phi)$ \citep{mohamed2020Monte}.
Two popular estimators are the score function estimator \citep{kleijnen1996Optimization,williams1992Simple}, and the reparameterized gradient estimator \citep{rezende2014Stochastic,kingma2014AutoEncoding,titsias2014Doubly}.
The latter provides lower variance gradient estimates but requires that the distribution $q(x; \phi)$ can be reparameterized and that $\gamma(x)$ is differentiable everywhere.

\section{Difficulties for Variational Inference in Universal PPSs}
\label{sec:difficulties}

The starting point for any VI scheme is to construct an appropriate variational family, also known as a guide.
\change{To automate inference, we desire to (at least partially) automate the process of constructing this guide.}
Existing methods for this all generate the guide on a variable-by-variable basis~\citep{wingate2013Automated,paige2016Automatic,vandemeent2018Introduction}: they introduce a single variational distribution $q_{a_i}(x; \phi_{a_i})$ for each unique sampling address $a_i$, then form the guide by replacing all the original random draws, $x_i\sim f_{a_i}(\,\cdot\mid\eta_i)$, with draws from the corresponding variational distribution instead.
This forms a global guide that maintains the stochastic dependency structure of the original program, such that the guide itself has stochastic support.

Our motivating insight is that this high-level approach has some fundamental limitations.
Consider the simple example from Fig.~\ref{fig:motivating_example}.
Here the variable \lstinline{x} influences the program's control flow.
This causes discontinuities that mean it is difficult to approximate its conditional density with a single variational approximation, especially if that variational approximation is restricted to a simple distribution class.
Here the different possible paths are essentially working against each other, as what is helpful for the approximation of \lstinline{x} on one path, is generally detrimental for the other.

A further complication occurs when the stochastic control flow of a program influences whether a variable exists at all.
Here it can become extremely challenging to set up guides which are faithful to the dependency structure of previously sampled variables~\citep{webb2018Faithful}, as the set of variables we condition on is itself stochastic.
Because of this, existing approaches typically rely on mean-field assumptions across paths~\citep{wingate2013Automated,paige2016Automatic,vandemeent2016BlackBox}.
However, this assumption is rarely reasonable given that the path typically strongly influences the distribution of individual variables.

Finally, creating a single unique variational approximation for each address also leads to challenging optimization problems: the same address can be present in multiple program paths, and the number of variables and their dependencies can vary between different paths.

\section{Support Decomposition Variational Inference}
\label{sec:methodology}

We now introduce a novel VI approach to overcome the challenges mentioned in Sec.~\ref{sec:difficulties}.
We call our method \emph{Support Decomposition Variational Inference} (SDVI), because the key design decision is the choice (and automatic construction) of a guide that takes the form of a mixture distribution over the set of possible paths a program can take.
That is, rather than constructing a single guide with the same stochastic control flow as the original program and a separate variational approximation for each \emph{unique address}, we instead construct separate sub-guides with deterministic control flows for each \emph{unique path}.
These are then combined into an overall guide using the mixture distribution which maximizes the overall ELBO.
As we will see, this alternative approach substantially simplifies the process of constructing effective guides and allows the full weight of the well-developed techniques for VI in the static support setting to be brought to bear on problems with stochastic support.

\subsection{Decomposing Probabilistic Programs into Straight-Line Programs}

As noted by, e.g.,~\citep{chaganty2013Efficiently,sankaranarayanan2013Static,zhou2020Divide}, all probabilistic programs can be reformulated as mixture distributions over \emph{straight-line programs} (SLPs), which are sub-programs without any stochastic control flow.
Building on our earlier notation, the constituent SLPs of a program correspond directly to the unique instances of the program path, $A$.
Given the path, the set of variables making up the raw random draws in the program is fixed, along with their form and reference distribution; that is each SLP represents a probabilistic model with fixed variable typing and  support.

Following the notation of~\citet{zhou2020Divide},
we can apply an arbitrary fixed ordering on the set of SLPs in a program, such that we can uniquely define and index them using the set of possible addresses $A_k$ for $k \in \mathcal{K}$, where $\mathcal{K}$ is a countable (but potentially infinite) indexing set.
Each SLP $A_k$ now corresponds to a particular sub-region, $\mathcal{X}_k$, of the raw random draw sample space, $\mathcal{X}$. 
These sub-regions are disjoint and their union is the full sample space.
Unlike $\mathcal{X}$, each element in any given $\mathcal{X}_k$ has the same length $n_k$ and is measurable with respect to the same reference measure.
The unnormalized density for the $k$th SLP is now given by
\begin{equation}
    \label{eq:slp_density}
    \gamma_{k}(x_{1:n_k}) 
    = \mathbb{I}[x_{1:n_k} \!\in\! \mathcal{X}_k] \, \gamma(x_{1:n_k}) 
    = \mathbb{I}[x_{1:n_k} \!\in\! \mathcal{X}_k] \prod\nolimits_{i=1}^{n_k} f_{A_k[i]}(x_i \mid \eta_i) \prod\nolimits_{j=1}^{n_{y}} g_{b_j}(y_j \mid \phi_j),
\end{equation}
and the unnormalized density function for the original program can be written as a simple sum of the individual SLP densities: $\gamma(x_{1:n_x}) = \sum_{k \in \mathcal{K}} \gamma_k(x_{1:n_x})$.
The corresponding normalized conditional density can then be written as mixture distribution over the conditional distributions of the individual SLPs, with mixture weights given by their (normalized) local partition functions:
\begin{equation}
    \label{eq:program_density_with_auxiliary}
\pi(x) = \sum_{k \in \mathcal{K}} \pi(x  \mid  k) \,  \pi(k) ~~ \text{where} ~~ \pi(x \mid k) = \frac{\gamma_{k}(x)}{Z_k}, \  \pi(k) = \frac{Z_k}{\sum_{\ell \in \mathcal{K}} Z_\ell}, \ Z_k = \int_{\mathcal{X}_k} \gamma_k(x) dx.
\end{equation}

\subsection{Decomposing the Variational Family into Straight-Line Programs}
\label{sec:define_kl}

The key idea behind SDVI is now to construct the guide using a factorization that is analogous to that of the SLP decomposition above.
Precisely, we aim to learn variational approximations of the form
\begin{equation}
\label{eq:guide_structure}
q(x;\phi,\lambda) = \sum\nolimits_{k=1}^K \qslp(x;\phi_k)q(k;\lambda)
\end{equation}
where $\qind(k;\lambda)$ defines a categorical distribution over the indices of the SLPs, with support $k \in \{1, \dots, K\}$; and $\qslp(x; \phi_k)$ is the local guide of the $k$th SLP, with support $x \in \mathcal{X}_k$. Critically, as each $\mathcal{X}_k$ represents a fixed support, the local variational families $q_k$ can be automatically constructed using standard techniques for static problems, as we discuss in Sec.~\ref{sec:local_guides}.
Note that it is valid for the guide $q(x;\phi,\lambda)$ to not cover all SLPs, i.e. it is possible that $K < |\mathcal{K}|$.

Writing $\phi = \{\phi_k\}_{k=1}^K$, the KL divergence we wish to minimize for standard VI is now
\begin{equation}
    \label{eq:global_program_kl}
    \KL(q(x; \phi, \lambda) \parallel  \pi(x)) = \mathbb{E}_{q(x; \phi, \lambda)} \left[ \log q(x; \phi, \lambda) - \log \pi(x) \right],
\end{equation}
which we call the \emph{global KL divergence}.
By standard reasoning, minimizing this is equivalent to maximizing the \emph{global ELBO}
\begin{equation}
\label{eq:global_elbo}
    \ELBO(\phi,\lambda) = \mathbb{E}_{q(x; \phi, \lambda)} \left[ \log \gamma(x) - \log q(x; \phi, \lambda) \right]
\end{equation}
which, as we show in App.~\ref{app:kl_derivation}, can be rewritten as
\begin{equation}
\label{eq:global_kl_in_terms_of_elbos}
\ELBO(\phi,\lambda)  = \mathbb{E}_{\qind(k; \lambda)} \left[ \ELBO_k(\phi_k) - \log \qind(k; \lambda) \right], ~~ \text{where} ~~~ 
\ELBO_k(\phi_k) := \mathbb{E}_{\qslp(x; \phi_k)} \left[ \log \frac{\gamma_{k}(x)}{\qslp(x; \phi_k)} \right]
\end{equation}
is the term we refer to as the \emph{local ELBO} for the $k$th SLP.
Notice that each $\ELBO_k(\phi_k)$ depends only on the parameter $\phi_k$ and the local SLP density $\gamma_k$; it is completely independent of $\qind(k; \lambda)$, the other SLPs, and $\phi_{k'}$ for $k' \neq k$.
Thus, it follows from~\eqref{eq:global_kl_in_terms_of_elbos} that the inference problem for the whole program can be decomposed into independent `local' inference problems for the component SLPs, along with establishing the mixture probabilities $q(k;\lambda)$.
Furthermore, it turns out that the optimal $\qind(k; \lambda)$ is simply the softmax of $\ELBO_1,\dots,\ELBO_K$, as shown by the following result.
\begin{restatable}{proposition}{optimalweights}
\label{prop:optimal_slp_weights}
Let $L = \{ \ELBO_1,\dots,\ELBO_K \}$ be the set of local ELBOs, defined as per~\eqref{eq:global_kl_in_terms_of_elbos}, where $L$ is countable but potentially not finite.
If $0<\sum_{k=1}^K \exp(\ELBO_k)<\infty$, then the optimal corresponding $q(k;\lambda)$ in terms of the global ELBO~\eqref{eq:global_elbo} is given by
\begin{equation}
    \label{eq:optimal_weights}
    \qind(k; \lambda) =  \exp(\ELBO_k) \, \big/ \, \sum\nolimits_{\ell = 1}^{K} \exp(\ELBO_{\ell}). 
\end{equation}
\end{restatable}
The proof of this result is given in App.~\ref{app:kl_derivation}.
Though each of the $\ELBO_k$ terms here is itself intractable, they can be estimated efficiently and accurately by simple Monte Carlo.
We can thus straightforwardly construct $q(k;\lambda)$ once we have learned our local variational approximations: noting that these two processes are separable, $q(k;\lambda)$ is not needed until \emph{after} the individual $q_k$ are trained.

\subsection{Finding SLPs}
\label{sec:finding}

We have just shown how we can solve the VI problem of a probabilistic program with stochastic support by reducing it to a set of \emph{independent} and \emph{simpler} VI problems, each concerning an SLP, a program with static support.
However, we still need a mechanism to `discover' the SLPs, i.e.~extract the possible address paths from a program.

Here we first note that we only need to consider an SLP if it has a non-zero probability of being identified under forward simulation of the program, while ignoring conditioning statements.
This hints at a cheap and simple discovery mechanism whereby we draw samples by forward simulation and take note of the unique paths that have been generated. \change{We can either do this upfront, or in an online manner whereby we seek new SLPs as our budget increases and we have scope to deal with them (see App.~\ref{app:details_res_alloc} for details).
Although this is a stochastic process that is not guaranteed to find all the SLPs for finite budgets, for the problems considered in our experiments, it was always able to
reliably identify all SLPs with \emph{non-negligible posterior mass}.
Nonetheless, this approach may not be sufficient for all problems, such as when the likelihood concentrates in an area of very low prior mass.  Here one should instead look to employ
more sophisticated discovery methods instead, such as those based on MCMC sampling \citep{zhou2020Divide} or static analysis of the program code \citep{chaganty2013Efficiently,nori2015R2,beutner2022guaranteed}.
}

\begin{algorithm}[t]
    \caption{Support Decomposition Variational Inference}
    \label{alg:vi_res_alloc}
    \begin{algorithmic}[1]
        \Require Target program $\gamma$, iteration budget $T$, minimum no. of SH candidates $m$
        \State Extract SLPs $\{\gamma_k\}_{k=1}^K$ from $\gamma$ and set $\mathcal{C}=\{1,\dots,K\}$
        \Comment{Sec~\ref{sec:finding}} \label{alg_line:slp_discovery}
        \State Formulate guide $q_k$ for each SLP and initialize parameters $\phi_k$ \Comment{Sec~\ref{sec:local_guides}}
        \For{$l = 1, \dots, L = \lceil \log_2(K) - \log_2(m) +1\rceil $}
            \For{$k \in \mathcal{C}$}
                \State 
                Perform $\left\lfloor \nicefrac{T}{L |\mathcal{C}|} \right\rfloor$ optimization iterations of $\phi_k$ targeting $\ELBO_{{\rm surr},k}(\phi_k)$
                \Comment{Sec~\ref{sec:local_guides}} \label{alg_line:slp_optimization}
            \EndFor
            \State Remove $\min(\lfloor|\mathcal{C}| / 2 \rfloor, |\mathcal{C}|-m)$ SLPs from $\mathcal{C}$ with the lowest $\change{\ELBO_{k}(\phi_k)}$
            \Comment{Sec.~\ref{sec:allocating}}
        \EndFor
        \item Truncate $q_k$ outside of SLP support, $\mathcal{X}_k$, using Eq.~\eqref{eq:constrain}
        \item Estimate each $\mathcal{L}_k(\phi_k)$ using Monte Carlo estimate of Eq.~\eqref{eq:global_kl_in_terms_of_elbos}
        \State Calculate $\qind(k; \lambda)$ according to Eq.~\eqref{eq:optimal_weights} and return $q(x;\phi,\lambda)$ as per Eq.~\eqref{eq:guide_structure}
    \end{algorithmic}
\end{algorithm}

\subsection{Allocating Resources}
\label{sec:allocating}

Using the same amount of computational budget on each SLP is potentially wasteful, particularly if there is a large number of SLPs with insignificant marginal likelihoods.
Therefore, we seek a scheme that allocates more computational resources to promising SLPs, making sure to exploit the fact that the different inference problems are trivially parallelizable.

To formalize this resource allocation problem, let $T$ represent some fixed resource budget.
Further, let $t_k$ be the amount of this budget we spend on optimizing the $k$th SLP, such that $\sum_k t_k = T$ at the end of our training.
Our ultimate aim is produce the maximum possible final global ELBO, which will be a function of $\phi_1(t_1),\dots,\phi_K(t_K)$, where $\phi_k(t_k)$ denotes the value of $\phi_k$ achieved after allocating $t_k$ resources to that SLP.
By plugging the optimal mixture distribution $\qind(k; \lambda)$ from~\eqref{eq:optimal_weights} into~\eqref{eq:global_elbo}, we see that, after some rearranging, our resource allocation can be formulated as trying to maximize
\begin{equation}
    \label{eq:res_alloc_problem}
   \ELBO(\phi,\lambda^*) = \log \sum\nolimits_{k=1}^{K} \exp(\ELBO_k(\phi_k(t_k))) 
    \quad \text{s.t. } \sum\nolimits_{k=1}^{K} t_k = T.
\end{equation}
In practice, this is not a suitable objective for controlling our resource allocation directly, as it is still itself a random variable given $t_1,\dots,t_K$, because the optimization procedure is stochastic.
Moreover, we cannot consider its expectation, since the distribution of the $\phi_k(t_k)$ is unknown.
However, it does provide insight into how we ideally would like to allocate resources: we want to allocate more resources to SLPs whose exponentiated ELBOs are significant.
In particular, we can think of the `reward' for allocating $\epsilon$ more resources to SLP $k$  as $\exp(\ELBO_k(\phi_k(t_k+\epsilon)))-\exp(\ELBO_k(\phi_k(t_k)))$.

One could now, in principle, formulate the problem as a sequential decision making problem~\citep{lattimore2020bandit}.
However, the diminishing nature of the rewards and the fact that they are highly unlikely to be sub-Gaussian, along with the need to allow choosing multiple arms at once for parallelization, mean that setting up such an approach which is effective in practice is likely to be quite challenging.

Instead, we propose a simple heuristic, based on the \emph{Successive Halving} algorithm (SH) \citep{karnin2013almost} (see App.~\ref{app:details_res_alloc} for a description), an approach commonly used for resource allocation in hyperparameter optimization (HO) \citep{li2018hyperband,falkner2018bohb}.
In standard SH, the final objective is to identify and train the single best candidate, whereas ours is to maximize the \emph{sum} of all the local ELBOs.
Despite this difference, the use of SH can still be justified by the fact that the distribution over $\exp(\mathcal{L}_k(\phi_k(\infty)))$ will typically be heavily concentrated to a small number of SLPs, often only a single one.
Nonetheless, we make a small adaptation to the approach to stop over-focusing on a single SLP: we stop the halving process when a chosen number, $1\le m \le K$, of the candidates are left, with $m=1$ corresponding to standard SH.
The minimum proportion of the budget allocated to any given candidate by this scheme is $1/(K\lceil\log_2 K-\log_2 m+1\rceil)$, so we can use $m$ as a hyperparameter to control how evenly resources are allocated, with $m=K$ corresponding to uniform allocation.
This approach is also helpful for parallelization, as we can set $m$ equal to the number of available cores.

Putting everything together, a summary of our SDVI algorithm is given in Algo.~\ref{alg:vi_res_alloc}.
In App.~\ref{app:details_res_alloc}, we further show how this can be extended to an online variant of the approach, wherein we repeatedly run SH using the objective $\exp(\alpha \mathcal{L}_k(\phi_k(t_k)))/t_k$, where $0<\alpha\le 1$ is a hyperparameter, with smaller values of $\alpha$ encouraging more exploration.

\subsection{Formulating and Training the Local Guides}
\label{sec:local_guides}

In Algo.~\ref{alg:vi_res_alloc} we assume a mechanism to construct the \emph{local guide} $\qslp(x; \phi_k)$ for each SLP specified by path $A_k$.
In many situations---notably when the program path is uniquely determined by the sampled values from discrete distributions---it is possible to construct guides $\qslp$ that are guaranteed to place support within the sub-region $\mathcal{X}_k$, which, in turn, allows us to use the reparameterized gradient estimator for the gradients of $\ELBO_k(\phi_k)$.
Many models encountered in practice, e.g. mixture models \citep{richardson1997Bayesian}, have this property.
In this case it is possible to eliminate all the variables which influence the control flow by conditioning, effectively setting them to constants; see App.~\ref{app:local_guide_details} for further details.

In situations where we cannot easily construct a $\qslp$ which places support only within $\mathcal{X}_k$, we need to take care when training our guide.
Recall that for path $A_k$ the number of variables $n_k$ and their type sequence is fixed, which allows us to construct an initial guide $\tildeqslp$ with correct dimensionality and variable typing.
Let the support of this guide be denoted by $\tildexcal_k = \text{supp}(\tildeqslp)$.
In general, we will have $\mathcal{X}_k \subset \tildexcal_k$, because the control flow in the program imposes additional constraints on each individual variable.
Having constructed a guide with $\text{supp}(\tildeqslp) = \tildexcal_k$, one might be tempted to optimize $\KL(\tildeqslp(x; \phi_k) \parallel \pi(x  \mid  k))$, but we cannot guarantee the absolute continuity condition (i.e.~$\tildeqslp(x; \phi_k) = 0$ if $\pi(x  \mid  k) = 0$), and so, the KL divergence may not be well-defined, giving an ELBO of $- \infty$. 
To alleviate this issue we temporarily create a new surrogate target density defined as
\begin{equation}
    \label{eq:surrogate_density}
    \tilde{\gamma}_{k}(x_{1:n_k}) := \gamma_k(x_{1:n_k}) + c \, \mathbb{I}[x_{1:n_k} \notin \mathcal{X}_k],
\end{equation}
for a small positive, finite constant $c$.
This surrogate density is used solely for optimizing $\tildeqslp(x; \phi_k)$.
We \change{train $\phi_k$ to} optimize the corresponding surrogate ELBO
\begin{equation}
    \label{eq:surr_elbo}
    \ELBO_{{\rm surr},k}(\phi_k) := \mathbb{E}_{\change{\tildeqslp}(x; \phi_k)} \left[\log \tilde{\gamma}_{k}(x) - \log \tildeqslp(x; \phi_k) \right].
\end{equation}
We need to be careful to choose an appropriate $c$ that is sufficiently small compared to the values of $\gamma_k(x)$ for $x \in \mathcal{X}_k$, which we ensure by setting $c$ adaptively.
During the SLP discovery phase (Line~\ref{alg_line:slp_discovery} in Algo.~\ref{alg:vi_res_alloc}), we keep track of the smallest density value encountered so far, and call that $d_{min}$.
We then set $c = 0.01 d_{min}$ to ensure that the density values for $\tilde{\gamma}_k(x)$ outside of $\mathcal{X}_k$ are significantly below the values of $\tilde{\gamma}_k(x)$ for $x \in \mathcal{X}_k$.
Hence, optimizing~\eqref{eq:surr_elbo} faithfully optimizes $\qslp$ to be a good approximation to $\gamma_k(x)$ while avoiding the issues of infinite ELBO values.
While $\tilde{\gamma}_k$ is not a proper unnormalized density (it will in general not integrate to a finite value) this is not an issue in practice due to the mode-seeking behaviour of optimizing the ELBO.

Unfortunately, the bounds on the support of the SLP inevitably create a discontinuity in the objective.  
Thus, for fully unbiased gradients we need to use the score function estimator or some extension thereof.  
However, in some cases, the bias of the reparameterization gradient estimator may be sufficiently small to warrant its use.
Note that $\ELBO_{{\rm surr},k}(\phi_k)$ retains the desirable property of the standard ELBO that, if the observations are conditionally independent given the latent variables, we can get unbiased estimates of the ELBO using minibatches of the full dataset \citep{hoffman2013Stochastica,kucukelbir2015Automatic,titsias2014Doubly}.

Further we need to be careful when initializing $\tildeqslp$ as we require it to place sufficient probability mass within $\mathcal{X}_k$ to provide a suitable training signal. 
To ensure this, we initialize \change{$\phi_k$} by  minimizing the \emph{forward} KL divergence between the prior density of the $k$th SLP and $\tildeqslp(x; \phi_k)$ %
\begin{align}
    \label{eq:prior_density}
    \KL(\pi_{prior,k}(x) \parallel  \tildeqslp(x; \phi_k)) \ 
    &\propto \  \mathbb{E}_{\pi_{prior}(x)} \left[ - \mathbb{I}[x \in \mathcal{X}_k] \log \tildeqslp(x; \phi_k) \right]
\end{align}
where $\pi_{prior}(x_{1:n_x}) := \prod_{i=1}^{n_x} f_{a_i}(x_i \mid \eta_i)$. 
This objective can be optimized via stochastic gradient descent (cf. App.~\ref{app:local_guide_details}). 
Note that, for the purpose of initialization, we are targeting the prior, and thus we do not have to resort to expensive schemes to estimate the gradients which are necessary if one aims to minimize the forward KL targeting the posterior \citep{naesseth2020markovian}.

So far we have outlined how to train $\tildeqslp$ but to evaluate the local ELBOs, $\ELBO_k$, we need to construct a distribution $\qslp$ which satisfies the hard constraint $\text{supp}(\qslp) = \mathcal{X}_k$.
Our solution for this is \emph{truncating} $\tildeqslp$ by checking whether specific raw random draws $x'_{1:n_k}$ are valid for the path $A_k$, i.e.~whether $\mathbb{I}\left[x'_{1:n_k} \in \mathcal{X}_k \right]$.
We can do this by simply executing the program with fixed draws set to $x'_{1:n_k}$ and then noting that the program terminates and follows the address path $A_k$ if, and only if, $x'_{1:n_k} \in \mathcal{X}_k$.
Thus, we truncate \change{$\tildeqslp$ using}
\begin{equation}
\label{eq:constrain}
\qslp(x; \phi_k) = \frac{\tildeqslp(x; \phi_k)\mathbb{I}\left[x \in \mathcal{X}_k \right]}{\tilde{Z}_{k}\change{(\phi_k)}}, ~~ \text{where} ~~ \tilde{Z}_{k}\change{(\phi_k)} = \int_{\tildexcal_k} \tildeqslp(x; \phi_k)\mathbb{I}\left[x \in \mathcal{X}_k \right] dx.
\end{equation}
Hence, $\qslp$ is implicitly defined as the output of a rejection sampler with $\tildeqslp$ as a proposal.
\change{Note, that as we use the surrogate ELBO in~\eqref{eq:surr_elbo} when training $\phi_k$, we never need to take gradients through $\qslp$ or $\tilde{Z}_{k}(\phi_k)$, thereby avoiding the significant practical issues this would cause (see App.~\ref{app:direct_training}).}
Thus, the local guide $\qslp$ (Eq.~\eqref{eq:constrain}) is only used for estimating the local ELBOs (Eq.~\eqref{eq:global_kl_in_terms_of_elbos}).
This is done by first drawing $N$ samples $\{x^{(i)}\}_{i=1}^N$ from $\tilde{q}_k$, then rejecting samples which do not fall into the SLP and estimate $\tilde{Z}_k$ as the acceptance rate of this sampler (i.e. $N_A/N$ where $N_A$ is the number of samples accepted).  
Using $A$ to denote the set of indices of accepted samples, we form our ELBO estimate as
\begin{equation}
\hat{\mathcal{L}}_k := \frac{1}{N_A} \sum\nolimits_{i \in A} \log(N_A \, \gamma_k (x^{(i)})) - \log( N \, \tilde{q}_k(x^{(i)};\phi_k)).
\end{equation}
Note here that $A$ and $N_A$ are random variables that both implicitly depend on $\phi_k$, which is why we can use this for estimation, but not training.

\section{Related Work}

The vast majority of prior work on deriving automated VI algorithms focuses on the setting of static support \citep{ranganath2014Black,kucukelbir2017automatic,rezende2015Variational,ambrogioni2021Automatic,dhaka2021Challenges,agrawal2020advances,lee2018Reparameterization}.
Of particular note,~\citep{webb2018Faithful,paige2016inference,stuhlmuller2013learning} also consider using variational families that do not match the dependency structure of the original problem, but they still require static support.
More generally, there have been models with stochastic support for which bespoke guides where developed which do not follow the control-flow structure of the input program \citep{eslami2016Attend}.
However, these custom guides do not leverage the breakdown of the input program into SLPs.

The Divide-Conquer-Combine (DCC) algorithm \citep{zhou2020Divide} also exploits the breakdown of the program density into individual SLPs. 
However, \citep{zhou2020Divide} mainly focused on local inference algorithms that are sampling based, especially MCMC.
As we showed in Sec.~\ref{sec:methodology} unique challenges and opportunities arise when we consider the breakdown from a variational perspective.
Further, our work shows that using a variational family based on SLPs naturally leads to divide-and-conquer style algorithm, due to the resulting separability of the ELBO.
One of the most practical differences is that SDVI only requires (exponentiated) ELBOs to be estimated for each SLP, rather than marginal likelihoods.
The former can typically be estimated substantially more accurately for a given budget, allowing SDVI to scale better to high dimensional problems (see Sec.~\ref{sec:gmm_experiment}).
\citep{chaganty2013Efficiently} and \citep{sankaranarayanan2013Static} both also use the general idea of breaking down programs into SLPs, but both papers consider starkly different problem settings. 
Neither have any direct link to variational inference.

Our work is situated in the larger context of automated inference for universal PPSs. 
Other popular approaches include particle-based methods \citep{paige2014Compilation,ge2018Turing,wood2014New,rainforth2016interacting,murray2018probabilistic} and MCMC approaches with automated proposals~\citep{wingate2011Lightweight,le2016Inference,mak2021Nonparametric,carpenter2017Stan}.
Some work has looked to perform \emph{amortized} inference over a range of possible datasets~\cite{le2017Inference,paige2016inference,stuhlmuller2013learning,harvey2019attention}, often by training a proposal that is similar to a variational approximation.

\section{Experiments}
\label{sec:experiments}

To make SDVI easily accessible to practitioners we have implemented it in Pyro with code available at \href{https://github.com/treigerm/sdvi\_neurips}{\nolinkurl{github.com/treigerm/sdvi\_neurips}}.
The first baseline we consider, \textbf{Pyro AutoGuide}, uses the \texttt{AutoNormalMessenger} class to automatically generate a guide, and trains it with Pyro's built-in tools for VI (\url{http://pyro.ai/examples/svi_part_i.html}).
As an additional VI baseline, we also implement a custom guide for each model which uses the variable-by-variable scheme outlined in Sec.~\ref{sec:difficulties}, in combination with the score function gradient estimator; 
we refer to this baseline as \textbf{BBVI}.
For SDVI, we run SH until there are 10 active SLPs left (i.e. $m = 10$ in Algo.~\ref{alg:vi_res_alloc}) and parallelize the computation across $10$ cores. 
We further construct each local guide distribution $\qslp$ as a mean-field normal.
The specific configurations for each method for each experiment are provided in App.~\ref{app:details_experiments}.

\subsection{Program with Normal Distributions}
\label{sec:numerical_example}

\begin{figure}[t]
    \centering
    \begin{subfigure}[t]{0.45\textwidth}
        \centering
        \includegraphics[width=\linewidth]{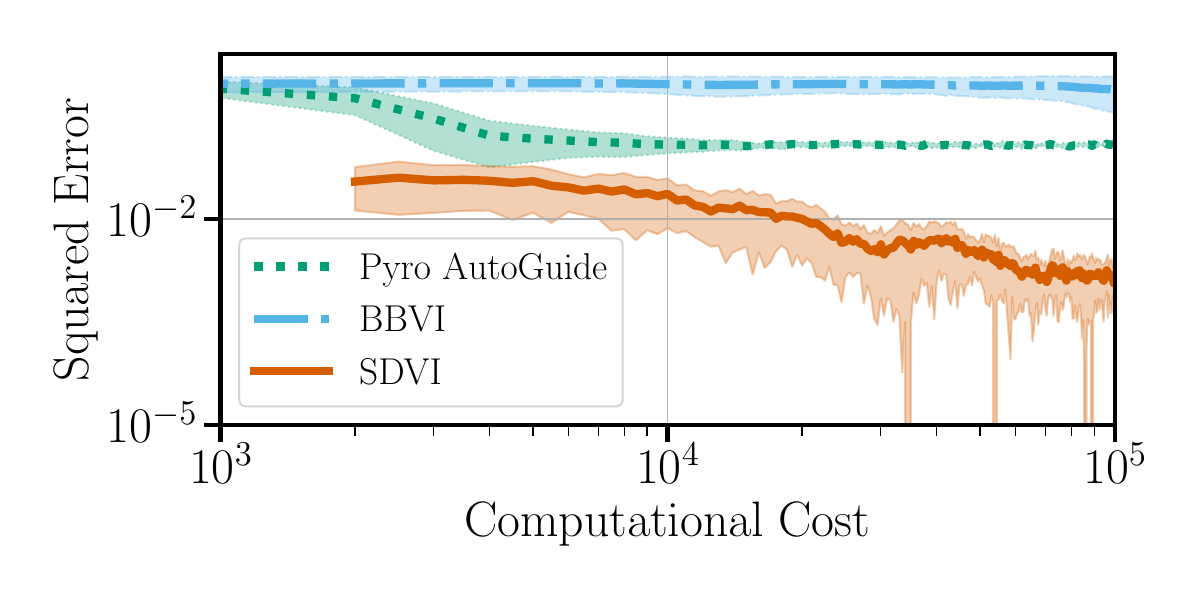}
         \vspace{-16pt}
        \caption{Squared error.}
        \label{fig:numerical_example_weights_error}
    \end{subfigure}
    \hfill
    \begin{subfigure}[t]{0.45\textwidth}
        \centering
        \includegraphics[width=\linewidth]{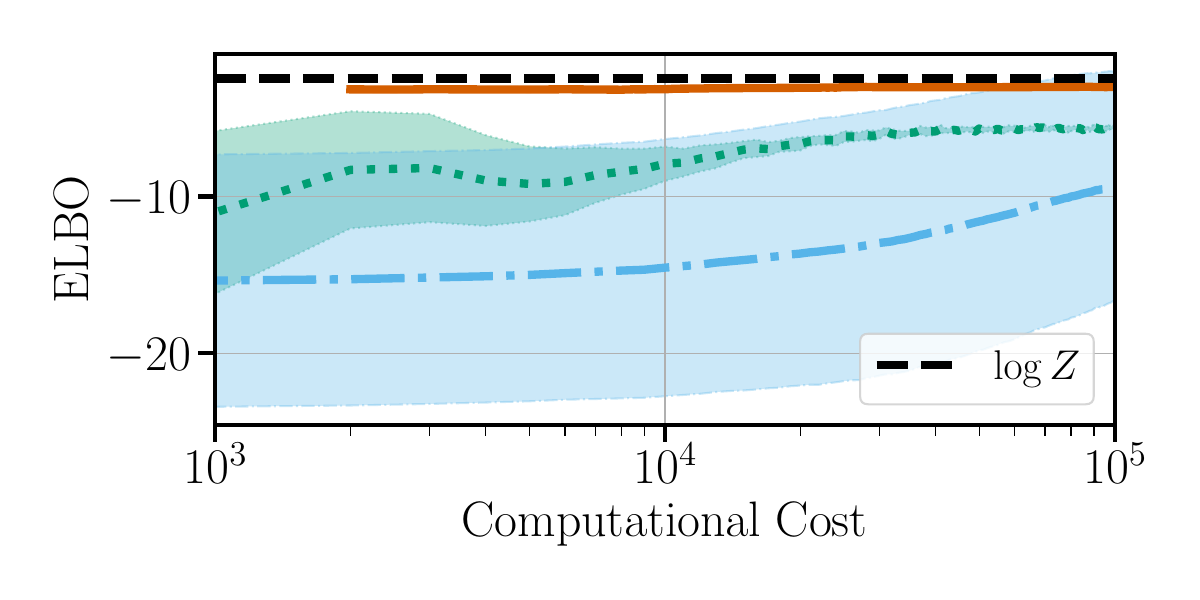}
        \vspace{-16pt}
        \caption{ELBOs.}
        \label{fig:numerical_example_elbos}
    \end{subfigure}
    \caption{Results for the model in §~\ref{sec:numerical_example}. Computational cost is measured in the number of likelihood evaluations. For each metric we show the mean and standard deviation over 10 runs. a) Squared error between the true SLP weights and the estimated SLP weights. b) Evidence Lower Bounds (ELBOs) for the variational algorithms, dashed line indicates the analytic log marginal likelihood.}
    \label{fig:numerical_example_results}
    \vspace{-10pt}
\end{figure}

We use our first experiment to further clarify the failure modes of existing VI approaches.
We consider an extension of the model from Fig.~\ref{fig:motivating_example} to contain more SLPs.
The full model is
\begin{equation}
    \label{eq:normal_model}
    \begin{split}
        u &\sim \mathcal{N}(0, 5^2), \\
        x &\sim \mathcal{N}(z, 1), \\
        y &\sim \mathcal{N}(x, 1).
    \end{split}
    \qquad
    \text{where}
    \qquad
    z = \begin{cases}
        0, &\text{if } u \in (-\infty, -4] \\
        K, &\text{if } u \in (-5 + K, -4 + K]  \text{ for } K = 1, \dots, 8 \\
        9, &\text{if } u \in (4, \infty) 
    \end{cases} 
\end{equation}
We assume we have observed $y=2$.
The results in Fig.~\ref{fig:numerical_example_results} demonstrate that SDVI is able to overcome the limitations of the other variational approaches.
\colorbbvi and \colorpyro both use the same guide in this model; \colorbbvi uses the score function gradient estimator for training, whereas \colorpyro uses the reparameterized gradient estimator.
This difference results in different posterior approximations for the different baselines.
The \colorbbvi guide tends to place all its mass on a single SLP and then provides a suitable approximation for only that one SLP, ignoring all the others.
This explains the large standard deviations for the ELBO values in Fig.~\ref{fig:numerical_example_elbos} as the ELBOs in different SLPs will converge to drastically different values.
For \colorpyro the biased gradient estimates will train the variational approximation for variable $u$ to be close to the prior $\mathcal{N}(0, 5^2)$.
\colorsdvi is able to avoid the shortcomings of the baselines as it provides an overall better posterior approximation leading to larger ELBO values, i.e.~lower KL divergences to the true posterior, and a more accurate weighting of the different SLPs (Fig.~\ref{fig:numerical_example_weights_error}).

\subsection{Infinite Gaussian Mixture Model}
\label{sec:gmm_experiment}

Our next model is a Gaussian Mixture Model (GMM) with a Poisson prior on the number of clusters:
\begin{equation*}
  K \sim \text{Poisson}(9) + 1; \quad  u_k \sim \mathcal{N}(\mathbf{0}, 10 \, \text{I}) \  \text{for } \; k = 1, \dots, K; \; \quad y \sim \frac{1}{K} \sum\nolimits_{k=1}^{K} \mathcal{N}(\mu_k, 0.1 \, \text{I}),
\end{equation*}
where $I$ is the $D \times D$ identity matrix and $\mathbf{0}$ is a $D$ dimensional vector of zeros (we set $D = 100$).
A similar model was considered in \citet{zhou2020Divide} but with $D=1$ instead of $D=100$.
We generate a dataset of 1250 observations with $K=5$.
To compare and evaluate the different algorithms, we hold out 250 data points as a test dataset to compute the log posterior predictive density (LPPD).

The Pyro AutoGuide baseline from the previous experiment is not applicable here since it assumes all latent variables are continuous.
In BBVI, for practical reasons, we had to cap the maximum number of clusters in the guide at 25 (cf. App.~\ref{app:details_experiments}).
To provide a further baseline, we have also implemented \textbf{DCC} \citep{zhou2020Divide} in Pyro with Random-walk lightweight Metropolis-Hastings (RMH) \citep{le2016Inference} as a local inference algorithm.
We chose DCC in particular because it also exploits the same breakdown into SLPs, so comparing against DCC is an opportunity to highlight the benefits of using a VI method.

In this model, the observations are assumed to be conditionally independent given the latent variables, thus enabling SDVI to work on subsets of the whole dataset \citep{hoffman2013Stochastica,kucukelbir2015Automatic}.
Specifically, we run SDVI on a model which samples a random minibatch of size $B=100$ at each iteration and then scales the likelihood by the factor $N / B$, where $N = 1000$ is the size of the full dataset; we refer to this setup as Stochastic SDVI (S-SDVI).
Furthermore, for this model SDVI is able to directly construct valid local guides $\qslp$ (using the mechanism for models branching on discrete variables outlined in Sec.~\ref{sec:local_guides}) and therefore (S-)SDVI can use the reparameterized gradient estimator.

Table~\ref{tab:gmm_results} shows that SDVI and S-SDVI significantly outperform the baselines, yielding a several orders of magnitude larger posterior predictive density and providing the only reasonable predictions for the numbers of clusters.
\change{In the few instances were (S-)SDVI returns a suboptimal MAP estimate of $K=6$, this was because the local guide for
the SLP with 5 components had fallen into a local model that fails to correctly identify all the clusters in the data, in turn returning a suboptimal local ELBO.}
BBVI and DCC struggle with this model due to the high-dimensional parameter space.
DCC's local inference algorithm, RMH, only updates one variable at a time which results in slow mixing times.
Note, DCC does not provide any ELBO values; its marginal likelihood estimator PI-MAIS \citep{martino2017Layered} constructs an importance sampling (IS) proposal distribution based on the outputs of MCMC chains which could theoretically be used to estimate an ELBO value.
However, as IS requires over-dispersed proposals, the ELBO scores for this approach are trivially $- \infty$, preventing a sensible comparison.

\begin{table}[t]
  \caption{Log posterior predictive densitiy (LPPD), ELBO, and maximum a posteriori (MAP) estimate for $K$ for GMM model. Mean and standard deviation for LPPD and ELBO computed over 5 runs.}
  \label{tab:gmm_results}
  \centering
  \begin{tabular}{llll}
    Method     & LPPD ($\uparrow, \times 10^3$) & ELBO ($\uparrow, \times 10^3$) & MAP $K$  \\
    \midrule
    DCC     & $-9842.90 \pm 3904.57$ & N/A & 14, 11, 16, 14, 15 \\
    BBVI    & $-2217.07 \pm 146.31$ & $-8770.55 \pm 544.95$ & 25, 25, 25, 25, 25 \\
    SDVI    & $\mathbf{32.84 \pm 0.02}$ & $\mathbf{128.76 \pm 0.17}$ & 5, 5, 6, 6, 5 \\
    S-SDVI    & $\mathbf{32.80 \pm 0.02}$ & $\mathbf{128.63 \pm 0.22}$ & 5, 5, 6, 5, 6 \\
    \bottomrule
  \end{tabular}
  \vspace{-5pt}
\end{table}

\subsection{Inferring Gaussian Process Kernels}
\label{sec:gp_experiment}

\begin{wrapfigure}{tr}{0.5\linewidth}
    \vspace{-48pt}
    \centering
    \includegraphics[width=\linewidth]{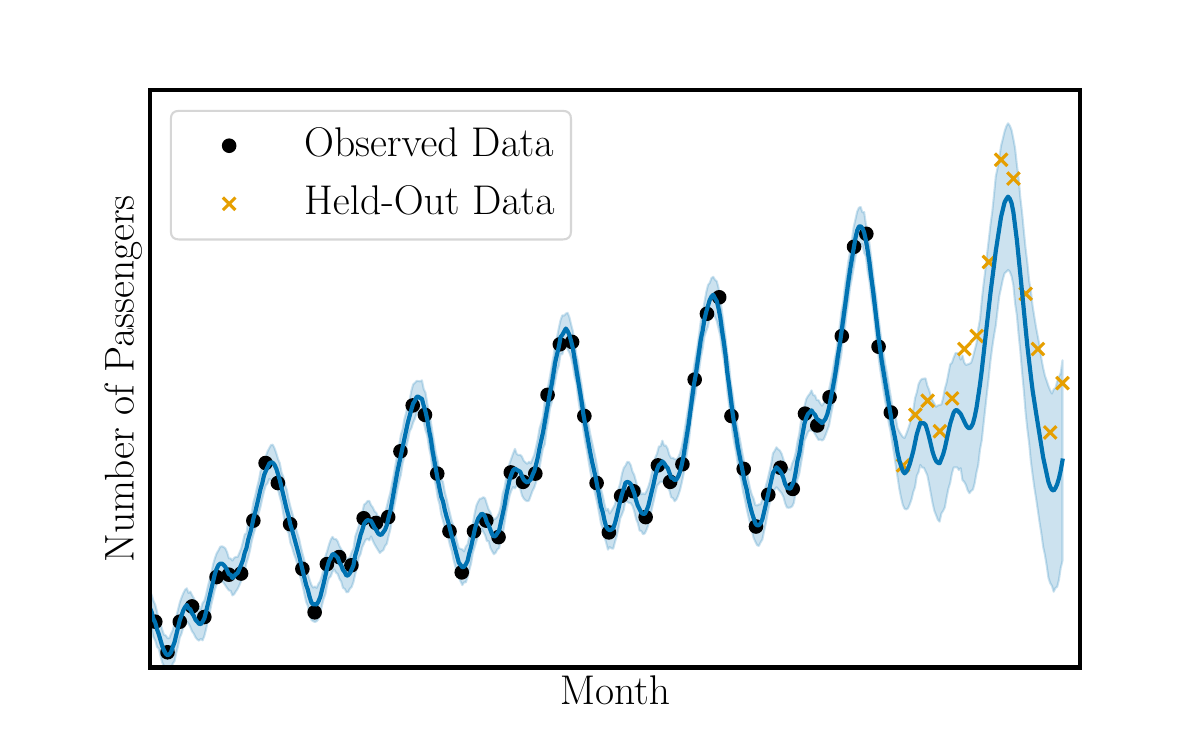}
    \caption{Posterior predictions of the GP for SDVI, shaded regions indicate 2 standard deviations that are computed from 100 posterior samples.}
    \label{fig:gp_predictions}
\end{wrapfigure}
For our final experiment, we consider the problem of inferring the kernel structure of a Gaussian Process (GP).
Following~\citep{duvenaud2013Structure,janz2016probabilistic}, we place a prior over kernel functions using a probabilistic context-free grammar (PCFG) .
We consider the squared exponential (SE), rational quadratic (RQ), periodic (PER), and linear (LIN) base kernels, and use the production rules
\begin{equation*}
    \mathcal{K} \rightarrow \  \text{SE} \  | \  \text{RQ} \  | \  \text{PER} \  | \  \text{LIN} \  | \  \mathcal{K} \times \mathcal{K} \  | \  \mathcal{K} + \mathcal{K}.
\end{equation*}
Sampling from the PCFG is implemented with a recursive probabilistic program that uses samples from a categorical distribution to decide which production rule in the PCFG should be applied.
In addition to the kernel structure, we also perform inference over the kernel hyperparameters for each base kernel and the observation noise; we place an inverse-gamma prior on each base kernel hyperparameter and a half-normal prior on the observation noise.
We further assume a normal likelihood function and marginalize out the latent GP.
Additional model details are in App.~\ref{app:details_experiments}.
We apply this model to a dataset of monthly counts of international airline passengers \citep{box2015time}, withholding the last 10 \% of all observations as a test dataset.

\begin{wraptable}{r}{0.5\linewidth}
  \vspace{-10pt}
  \caption{Final log posterior predictive densitiy (LPPD) and ELBO for GP model. Shown are mean and standard deviation computed over 5 runs.}
  \label{tab:gp_results}
  \centering
  \begin{tabular}{lll}
    Method     & LPPD ($\uparrow$) & ELBO ($\uparrow$)  \\
    \midrule
    DCC     & $-58.92 \pm 32.47$ & N/A \\
    BBVI    & $-18.82 \pm 1.20$ & $-48.48 \pm 0.33$ \\
    SDVI    & $\mathbf{2.05 \pm 3.30}$ & $\mathbf{34.53 \pm 21.42}$     \\
    \bottomrule
  \end{tabular}
  \vspace{-10pt}
\end{wraptable}
For SDVI we can construct local valid proposals $\qslp$ using the mechanism for models with discrete branching outlined in Sec~\ref{sec:local_guides}.
Hence, in each SLP the local guide $\qslp$ provides a posterior approximation over the kernel hyperparameters and the observation noise; the posterior distribution over kernel structures is implicitly defined through the mixture distribution over program paths.
Table~\ref{tab:gp_results} shows that SDVI provides higher LPPD values, 
and is also able to achieve a higher final ELBO value compared to BBVI.
Fig.~\ref{fig:gp_predictions} shows the posterior predictions for the SDVI run with the median LPPD score.
SDVI is able to provide qualitatively reasonable predictions, as the predictions follow the periodic trend in the observed data.

\section{Discussion}

We believe that SDVI provides a number of significant contributions towards the goal of effective (automated) inference for probabilistic programs with stochastic support, nonetheless it still naturally has some limitations. 
Perhaps the most obvious is that it, if there is a very large number of SLPs that cannot be easily discounted from having significant posterior mass, it can be challenging to learn effective variational approximations for all of them, such that SDVI is likely to perform poorly if the number becomes too large. 
Here, customized conventional VI or reversible jump MCMC approaches might be preferable, as they can be set up to focus on the transitions between SLPs, rather than trying to carefully characterize individual SLPs.

Another limitation is that our current focus on automation means that there are still open questions about how best to construct more customized guides within the SDVI framework. Here the breakdown into individual SLPs and use of resource allocation strategies will still often be useful, but changes to our implementation would be required to allow more user control and customization. 
For example, the discovery of individual SLPs using the prior is a potential current failure mode, and it would be useful to support the use of more sophisticated program analysis techniques (e.g. \citep{beutner2022guaranteed}).

A more subtle limitation is that the local inferences of each SLP can sometimes still be quite challenging themselves. 
If the true posterior places a lot of mass near the boundaries of the SLP, there can still be a significant posterior discontinuity, meaning we might need advanced local variational families (e.g. normalizing flows) and/or gradient estimators. 
Such problems also occur in static support settings and are usually much more manageable than the original stochastic support problem, but further work is needed to fully automate dealing with them.

Finally, variational methods are often used not only for inference, but as a basis for model learning as well. 
In principle, SDVI could also be used in such settings, but as described in App.~\ref{app:parameter}, there are still some hurdles that need to be overcome to do this in practice.

\vspace{-6pt}
\section{Conclusion}
\vspace{-3pt}

We have presented SDVI and shown that it is able to overcome the limitations of existing VI approaches for programs with stochastic support by using a novel guide structure that breaks the program down into SLPs with fixed support, rather than matching the original stochastic control flow.
The structure of the variational family separates the ELBO into multiple independent inference problems which naturally motivates a divide-and-conquer style training procedure with explicit resource allocation.
Experimentally we found that these innovations meant that SDVI was able to provide significant performance improvements over the previous state-of-the-art approaches.

\begin{ack}
We would like to thank Yuan Zhou for useful discussions in the early stages of this project.
Tim Reichelt is supported by the UK EPSRC CDT in Autonomous 
Intelligent Machines and Systems with the grant EP/S024050/1.
Luke Ong would like to acknowledge funding from EPSRC UK and National Research Foundation Singapore NRF-RSS2022-009.
\end{ack}

\bibliographystyle{unsrtnat}
\bibliography{references}

\newpage
\section*{Checklist}

\begin{enumerate}

\item For all authors...
\begin{enumerate}
  \item Do the main claims made in the abstract and introduction accurately reflect the paper's contributions and scope?
    \answerYes{}
  \item Did you describe the limitations of your work?
    \answerYes{Limitations are discussed in relevant sections throughout the paper.}
  \item Did you discuss any potential negative societal impacts of your work?
    \answerNA{The paper provides a generic variational inference algorithm for probabilistic models with stochastic support. As our contributions are methodological and our experiments do not contain any personalized data, we believe that this paper does not introduce any fundamentally new risks.}
  \item Have you read the ethics review guidelines and ensured that your paper conforms to them?
    \answerYes{}
\end{enumerate}

\item If you are including theoretical results...
\begin{enumerate}
  \item Did you state the full set of assumptions of all theoretical results?
    \answerYes{See Proposition~\ref{prop:optimal_slp_weights}.}
	\item Did you include complete proofs of all theoretical results?
    \answerYes{See Appendix~\ref{app:kl_derivation}.}
\end{enumerate}

\item If you ran experiments...
\begin{enumerate}
  \item Did you include the code, data, and instructions needed to reproduce the main experimental results (either in the supplemental material or as a URL)?
    \answerYes{See supplementary material.}
  \item Did you specify all the training details (e.g., data splits, hyperparameters, how they were chosen)?
    \answerYes{See Section~\ref{sec:experiments} and Appendix~\ref{app:details_experiments}.}
	\item Did you report error bars (e.g., with respect to the random seed after running experiments multiple times)?
    \answerYes{}
	\item Did you include the total amount of compute and the type of resources used (e.g., type of GPUs, internal cluster, or cloud provider)?
    \answerYes{See Appendix~\ref{app:details_experiments}.}
\end{enumerate}

\item If you are using existing assets (e.g., code, data, models) or curating/releasing new assets...
\begin{enumerate}
  \item If your work uses existing assets, did you cite the creators?
    \answerYes{}
  \item Did you mention the license of the assets?
    \answerNA{}
  \item Did you include any new assets either in the supplemental material or as a URL?
    \answerYes{Code in supplementary material.}
  \item Did you discuss whether and how consent was obtained from people whose data you're using/curating?
    \answerNA{}
  \item Did you discuss whether the data you are using/curating contains personally identifiable information or offensive content?
    \answerNA{}
\end{enumerate}

\item If you used crowdsourcing or conducted research with human subjects...
\begin{enumerate}
  \item Did you include the full text of instructions given to participants and screenshots, if applicable?
    \answerNA{}
  \item Did you describe any potential participant risks, with links to Institutional Review Board (IRB) approvals, if applicable?
    \answerNA{}
  \item Did you include the estimated hourly wage paid to participants and the total amount spent on participant compensation?
    \answerNA{}
\end{enumerate}

\end{enumerate}

\clearpage

\appendix

\thispagestyle{empty}
\rule{\textwidth}{1pt}
\vspace{-6pt}
\begin{center}
	\textbf{ \Large Appendix for Rethinking Variational Inference for Probabilistic Programs with Stochastic Support}
\end{center}\vspace{-6pt}
\rule{\textwidth}{1pt}
\begin{minipage}{\textwidth}
	\centering
	\vspace{17pt}
	\textbf{Tim Reichelt}\textsuperscript{1} \qquad \textbf{Luke Ong}\textsuperscript{1,2} \qquad \textbf{Tom Rainforth}\textsuperscript{1}\\
  \textsuperscript{1} University of Oxford \\
  \textsuperscript{2} Nanyang Technological University, Singapore \\
  \texttt{\{tim.reichelt,lo\}@cs.ox.ac.uk \quad rainforth@stats.ox.ac.uk}
	\vspace{-2pt}
\end{minipage}
\vspace{10pt}

\section{KL Divergence Derivation}
\label{app:kl_derivation}

\todo{Ensure all the notation is consistent with the main paper.}

\subsection{Breaking Down the Global ELBO}

The global ELBO is given by 
\begin{align}
    \ELBO(\phi,\lambda) &= \mathbb{E}_{q(x; \phi, \lambda)} \left[ \log \frac{\gamma(x)}{q(x; \phi, \lambda)} \right], \\
    &= \int_{\mathcal{X}} q(x; \phi, \lambda) \log \frac{\gamma(x)}{q(x; \phi, \lambda)} dx, \\
\intertext{using the fact that the subsets $\mathcal{X}_k$ provide a partition of $\mathcal{X}$ we can write the integral as}
    &= \sum_{k=1}^{K} \int_{\mathcal{X}_k} q(x; \phi, \lambda) \log \frac{\gamma(x)}{q(x; \phi, \lambda)} dx, \\
\intertext{using the factorization of $q(x; \phi, \lambda)$ and the fact that for $x \in \mathcal{X}_k$ the program density satisfies $\gamma(x) = \gamma_k(x)$ we get}
    &= \sum_{k=1}^{K} \int_{\mathcal{X}_k} \qslp(x; \phi_k) \qind(k; \lambda) \log \frac{\gamma_k(x)}{\qslp(x; \phi_k) \qind(k; \lambda)} dx \\
\intertext{then using the fact that $\qind(k; \lambda)$ does not depend on $x$ we have}
    &= \sum_{k=1}^{K} \qind(k; \lambda) \int_{\mathcal{X}_k} \qslp(x; \phi_k)  \log \frac{\gamma_k(x)}{\qslp(x; \phi_k)} dx - \log \qind(k; \lambda), \\
\intertext{which we can write concisely as}
    &= \mathbb{E}_{\qind(k; \lambda)} \left[ \ELBO_k(\phi_k) - \log \qind(k; \lambda) \right],
\end{align}
where
\begin{equation*}
    \ELBO_k(\phi_k) := \mathbb{E}_{\qslp(x; \phi_k)} \left[ \log \frac{\gamma_{k}(x)}{\qslp(x; \phi_k)} \right].
\end{equation*}

\subsection{Optimal Setting of $\qind(k; \lambda)$}

\optimalweights*
\begin{proof}
By the assumption that $0<\sum_{k=1}^K \exp (\ELBO_k)<\infty$, we have that $\nicefrac{\exp(\ELBO_k)}{\sum_{k=1}^K \exp (\ELBO_k)}$ forms a valid probability mass function over $k\in\{1,\dots,K\}$.
We can therefore rewrite ~\eqref{eq:global_kl_in_terms_of_elbos} as
\begin{align}
    \ELBO(\phi,\lambda) &= \mathbb{E}_{\qind(k; \lambda)} \left[ \log \frac{\exp(\ELBO_k)}{\sum_{k=1}^K \exp (\ELBO_k)} - \log \qind(k; \lambda) \right]
    +\log \sum_{k=1}^K \exp (\ELBO_k) \\
    &= -\KL\left(\qind(k; \lambda)  \parallel \frac{\exp(\ELBO_k)}{\sum_{k=1}^K \exp (\ELBO_k)}\right) + \log \sum_{k=1}^K \exp (\ELBO_k) \label{eq:optimal_weights_proof}
\end{align}
Now as second term in the above is constant in $q(k;\lambda)$ and a KL divergence is minimized when the two distributions are the same, we can immediately conclude the desired result that the optimal $q(k;\lambda)$ is 
\begin{equation}
    \qind(k; \lambda) =  \frac{\exp(\ELBO_k)}{\sum_{\ell = 1}^{K} \exp(\ELBO_{\ell})}.
\end{equation}

\end{proof}

Additionally, from~\eqref{eq:optimal_weights_proof} it follows that for the optimal setting of the mixture distribution $\qind(k; \lambda)$ the global ELBO is given by
\begin{equation*}
    \ELBO(\phi,\lambda^*) = \log \sum_{k=1}^K \exp (\ELBO_k).
\end{equation*}

\section{Details on Resource Allocation}
\label{app:details_res_alloc}

\subsection{Background on Successive Halving}
\label{app:sh}

Successive Halving (SH) divides a total budget of $T$ iterations into $L = \lceil \log_2(K) \rceil+1$ phases and starts by optimizing each of $K$ candidates, in our case the SLPs, for $\lfloor T/(KL) \rfloor$ iterations. 
It then ranks each of the candidates in terms of their performance, in our case the values of $\exp(\mathcal{L}_k)$, before eliminating the bottom half.
This process then repeats, with each of the remaining candidates run for $2^{\ell-1} T/(KL)$ iterations at the $\ell$-th phase.
This results in an exponential distribution of resources allocated to the different candidates, with more resources allocated to those that are more promising after intermediate evaluation.

Adapting it to our setting of treating the problem as a top-$m$ identification is done by simply using $L=\lceil \log_2(K) -\log_2(m)\rceil+1$ phases instead of $L=\lceil \log_2(K)\rceil+1$.

\subsection{Online Resource Allocation}
\label{app:principled_res_alloc}

Here, we present an online version of Algo.~\ref{alg:vi_res_alloc}, where the term `online' refers to the fact that the algorithm considers more and more SLPs as the computational budget increases.
The online variant of the algorithm is useful if a user is unsure about the total iteration budget that they want to spend on the input program.
This user might want to run SDVI with an initial iteration budget $T_1$ and after having observed the results, they might decide that they want to keep further optimizing the guide parameters.
We therefore need to adapt Algo.~\ref{alg:vi_res_alloc} so that it can be `restarted' after it has terminated.
A naive approach to this would be to simply run Algo.~\ref{alg:vi_res_alloc} again but re-use the $\qslp$'s for the SLPs that have already been discovered and only initialize the $\qslp$ from scratch for SLPs which have not been seen before.
However, this scheme is limited as it disproportionately favours SLPs which were discovered in the previous run.
This is because for those SLPs the local ELBOs will already be relatively large compared to the newly added SLPs.
As a consequence, SH will not assign significant computational budget to the SLPs that were added after the algorithm was restarted.

To safeguard against this behaviour we instead propose an online version of SDVI in Algo.~\ref{alg:online_res_alloc} which is using a modified `reward' for SH.
Instead of ranking the different SLPs according to $\ELBO_k(\phi_k(t_k))$ we instead propose the objective $\exp(\alpha \ELBO_k(\phi_k(t_k))) / t_k$ where $0 < \alpha \le 1$.
The reward is scaled by the reciprocal of $t_k$ because we are no longer aiming to select the SLPs with the highest $\ELBO_k(\phi_k(t_k))$ but instead aim to choose the SLPs which have been `underselected' compared to other SLPs, assuming we should have selected them in proportion to $\exp(\alpha \ELBO_k(\phi_k(t_k)))$.
The scaling by the scalar $\alpha$ is a further mechanism to encourage more exploration, with setting $\alpha = 0$ equivalent to uniform sampling in the limit of repeated SH runs.
Since this adapted objective takes into account the computational budget that was spent on each SLP, it is a more suitable objective when running SH repeatedly.

\begin{algorithm}[H]
    \caption{Online SDVI}
    \label{alg:online_res_alloc}
    \begin{algorithmic}[1]
        \Require Target program $\gamma$, iteration budget per SH run $T$, minimum no. of SH candidates $m$, parameter controlling $\alpha > 0$ exploration 
        \State Extract SLPs $\{\gamma_k\}_{k=1}^K$ from $\gamma$ and set $\mathcal{C}=\{1,\dots,K\}$
        \State Formulate guide $q_k$ for each SLP and initialize parameters $\phi_k$
        \State $t_k = 0$ for all $k \in \mathcal{C}$
        \While{Stopping criteria not satisfied}
            \State $\mathcal{C}' \leftarrow \mathcal{C}$
            \State Phases in successive halving $L = \lceil \log_2(|\mathcal{C}|) - \log_2(m) \rceil + 1$
            \For{$l = 1, \dots, L$}
                \State Number of iterations $n_l = \lfloor \frac{T}{L |\mathcal{C}'|} \rfloor$
                \For{$k \in \mathcal{C}'$}
                    \State Perform $n_l$ optimization iterations of $\phi_k$ targeting $\ELBO_{{\rm surr},k}(\phi_k)$
                    \State Estimate $\ELBO_{{\rm surr},k}(\phi_k)$ using Monte Carlo estimate of Eq.~\eqref{eq:surr_elbo}
                    \State $t_k = t_k + n_l$
                \EndFor
                \State Remove $\min(\lfloor|\mathcal{C}'| / 2 \rfloor, |\mathcal{C}'|-m)$ SLPs from $\mathcal{C}'$ with the lowest $\exp(\alpha \, \ELBO_{{\rm surr},k}(\phi_k)) / t_k$
            \EndFor
            \State Extract new SLPs from $\gamma$ and add them to $\mathcal{C}$, set $t_{k'} = 0$ for each new SLP with index $k'$
        \EndWhile
        \item Truncate $q_k$ outside of SLP support, $\mathcal{X}_k$, using Eq.~\eqref{eq:constrain}
        \item Estimate each $\mathcal{L}_k(\phi_k)$ using Monte Carlo estimate of Eq.~\eqref{eq:global_kl_in_terms_of_elbos}
        \State Calculate $\qind(k; \lambda)$ according to Eq.~\eqref{eq:optimal_weights} and return $q(x;\phi,\lambda)$ as per Eq.~\eqref{eq:guide_structure}
    \end{algorithmic}
\end{algorithm}

\section{Details for Training Local Guides}
\label{app:local_guide_details}

\subsection{Density Estimation of the Prior}

\todo{Note the fact that the term prior distribution is not completely appropriate since the $\eta_i$ might depend on observed data.}
Before we can define the KL divergence we first have to carefully define global and local prior distributions
We first define what we informally call the global `prior' distribution of the program as the product of all the terms added to the program density by the \lstinline{sample} statements
\begin{equation}
    \label{eq:global_prior_def}
    \pi_{\mathrm{prior}}(x_{1:n_x}) := \prod_{i=1}^{n_x} f_{a_i}(x_i|\eta_i).
\end{equation}
However, here we are using the term prior only informally, since~\eqref{eq:global_prior_def} is not a prior in the conventional Bayesian sense since the $\eta_i$ can be functions of the observed data $y$.
Note that here $n_x$ in~\eqref{eq:global_prior_def} is again a random variable since the raw random draws $x_{1:n_x}$ of the program do not necessarily have fixed length.
Then similarly we define local `prior' distributions
\begin{equation}
    \label{eq:local_prior_def}
    \pi_{\mathrm{prior},k}(x_{1:n_k}) := \frac{\mathbb{I}[x_{1:n_k} \in \mathcal{X}_k] \prod_{i=1}^{n_k} f_{A_k[i]}(x_i|\eta_i)}{Z_{\mathrm{prior},k}} = \frac{\mathbb{I}[x_{1:n_k} \in \mathcal{X}_k] \pi_{\mathrm{prior}}(x_{1:n_k})}{Z_{\mathrm{prior},k}}, ,
\end{equation}
where
\begin{equation}
    Z_{\mathrm{prior},k} := \int_{\mathcal{X}} \mathbb{I}[x \in \mathcal{X}_k] \, \pi_{\mathrm{prior}}(x) dx.
\end{equation}
Note that for our purposes we will never actually have to estimate $Z_{prior,k}$, we only defined it to ensure that $\pi_{prior, k}$ is a normalized density.
This allows us to define the forward KL divergence which we would like to optimize with respect to $\phi_k$
\begin{align}
    \KL(\pi_{\mathrm{prior},k}(x) \parallel  \tildeqslp(x; \phi_k)) 
    &= \mathbb{E}_{\pi_{\mathrm{prior},k}(x)} \left[ \log \frac{\pi_{\mathrm{prior},k}(x)}{ \tildeqslp(x; \phi_k)} \right] \\ 
    \intertext{which we can rewrite as}
    &= \mathbb{E}_{\pi_{\mathrm{prior},k}(x)} \left[ \log \pi_{\mathrm{prior},k}(x) \right] - \mathbb{E}_{\pi_{\mathrm{prior},k}(x)} \left[ \log \tildeqslp(x; \phi_k) \right]. \\
    \intertext{The first term is a constant with respect to $\phi_k$ and therefore does not affect the optimization}
    &\propto \  \mathbb{E}_{\pi_{\mathrm{prior},k}(x)} \left[ - \log \tildeqslp(x; \phi_k) \right], \\
    \intertext{then by the definition of $\pi_{\mathrm{prior},k}(x)$ in Eq.~\eqref{eq:local_prior_def} this is equivalent to}
    &= - \frac{1}{Z_{\mathrm{prior},k}} \mathbb{E}_{\pi_{\mathrm{prior}}(x)} \left[ \mathbb{I}[x \in \mathcal{X}_k] \log \tildeqslp(x; \phi_k) \right]. \\
    \intertext{Finally, $Z_{prior,k}$ is a constant with respect to $\phi_k$ and can be dropped}
    &\propto\  \mathbb{E}_{\pi_{\mathrm{prior}}(x)} \left[ - \mathbb{I}[x \in \mathcal{X}_k] \log \tildeqslp(x; \phi_k) \right].
    \label{eq:app_prior_density_est}
\end{align}

We can estimate the gradients of the objective in Eq.~\eqref{eq:app_prior_density_est} using a Monte Carlo estimator 
\begin{equation}
    \nabla_{\phi_k} \mathbb{E}_{\pi_{\mathrm{prior}}(x)} \left[ - \mathbb{I}[x \in \mathcal{X}_k] \log \tildeqslp(x; k, \phi_k) \right] \approx \frac{1}{N} \sum^N_{j=1}  \mathbb{I}[x^{(j)} \in \mathcal{X}_k] \nabla_{\phi_k} \log \tildeqslp(x^{(j)}; k, \phi_k)
\end{equation}
where $x^{(j)}$ are raw random draws generated by executing the input program forward.
These gradient estimates can then be used in a stochastic gradient descent optimization procedure.
In our experiments, we generate a fixed set of $N$ samples and re-use the same set of samples for the entire optimization process.
Other approaches are also possible such as periodically collecting a new set of samples and using local MCMC moves to collect samples instead of repeatedly sampling from the prior.

\subsection{Exploiting Program Structure: Discrete Branching Optimization}
\label{sec:exploit_structure}

In practice, many user-defined programs have structural properties which can be exploited to construct a valid local guide directly and deterministically (without resorting to the stochastic mechanism described in Sec.~\ref{sec:local_guides}).
Specifically, consider the class of programs whose program paths are determined by variables sampled from discrete distributions.
For these programs, we can assume that for each SLP ($k$th, say) there is an (ordered) set of indices $\branchind \subset \{1, \dots, n_k\} = I$ and a set of constants $r_{k,1}, \dots, r_{k, |\branchind|} \in \mathbb{Z}$ such that the local unnormalized densities are expressible as
\begin{equation*}
    \gamma_{k}(x_{1:n_k}) = \gamma(x_{1:n_k}) \prod_{l=1}^{|\branchind|} \mathbb{I} \left[ x_{\branchind[l]} = r_{k,l} \right]
\end{equation*}
where $\branchind[j]$ means the $j$th element in $\branchind$.
It follows that we can construct densities for the $k$th SLP on a subset of variables in $x_{1:n_k}$ by eliminating all the variables given by indices $\branchind$ (by instantiating them to constants). 
This is effectively equivalent to replacing the \lstinline{sample} statements corresponding to the variables which influence the control flow with \lstinline{observe} statements which induces a new program density that has the form
\begin{equation}
    \label{eq:discrete_branching_density}
    \tilde{\gamma}_k(x_{1:n'_{k}}) = \prod_{i = 1}^{n'_{k}} f_{A_k[\nonbranchind[i]]}(x_i|\eta_i) \prod_{l=1}^{|\branchind|} f_{A_k[\branchind[l]]}(r_{k,l} \mid \eta_{l})  \prod_{j=1}^{n_{y}} g_{b_j}(y_j \mid \phi_j)
\end{equation}
where $\nonbranchind := \left[1, \dots, n_k \right] \setminus \branchind$, %
and $n'_k := |\nonbranchind|$.
Furthermore, if all the remaining r.v.~are continuous distributions with support in $\mathbb{R}$ (i.e. $\text{supp}(f_{A_k[i]}) = \mathbb{R}$ for $i \in \nonbranchind$) then $\tilde{\gamma}_k(x_{1:n'_{k}})$ itself has support in $\mathbb{R}^{n'_k}$.
It is then straightforward to construct a guide $\qslp$ with support in $\mathbb{R}^{n'_k}$ using existing methods, and we can get gradient estimates using the reparameterization gradient estimator (assuming there are no more discontinuities in $\tilde{\gamma}_k$).

\todo{Explain discrete branching annotations.}

To realize the discrete branching optimization in our Pyro implementation we allow users to annotate the \lstinline{sample} statements which influence the branching.
While it is in principle possible to automatically identify programs with discrete branching using program analysis, formalizing and implementing such a program analysis tool to work with arbitrary Pyro program would be a significant contribution in itself which is out of scope for this paper as we are focused on the statistical evaluation of SDVI.
Specifically, the relevant \lstinline{sample} statements within a Pyro program can be annotated as follows: \lstinline|pyro.sample("x", dist.Poisson(7), infer={"branching": True})|.
Our implementation of SDVI is then able to use these annotations to create the density $\tilde{\gamma}_k$ in~\eqref{eq:discrete_branching_density}.

\section{Additional Details for Experiments}
\label{app:details_experiments}

For all experiments that rely on optimization we use the Adam optimizer \citepsupp{kingma2014adam}.
The experiments were executed on an internal cluster which uses a range of different computer architectures.

\subsection{Model From Figure~\ref{fig:motivating_example}}

\begin{lstlisting}[numbers=none,caption={Pyro Code for Figure~\ref{fig:motivating_example}.},label={lst:motivating_example_pyro}]
import pyro
import pyro.distributions as dist

def model():
    x = pyro.sample("x", dist.Normal(0, 1))
    if x < 0:
        z1 = pyro.sample("z1", dist.Normal(-3, 1))
    else:
        z1 = pyro.sample("z2", dist.Normal(3, 1))

    x = pyro.sample("x", dist.Normal(z1, 2), obs=torch.tensor(2.0))
 
guide = pyro.infer.autoguide.AutoNormalMessenger(model)
optim = pyro.optim.Adam({"lr": 0.01})
svi = pyro.infer.SVI(
    model, guide, optim, loss=pyro.infer.Trace_ELBO()
)

for j in range(2000):
    svi.step()
\end{lstlisting}

The full Pyro code for the model in Fig.~\ref{fig:motivating_example}, including automatically generating and training the guide is given in Listing~\ref{lst:motivating_example_pyro}.
The code for BBVI and SDVI is provided in the code supplementary.
For Pyro's AutoGuide and BBVI we run the optimization for $2000$ iterations with a learning rate of $0.01$.
Similarly, for SDVI we have a total iteration budget of $T = 2000$ and use a learning rate of $0.01$; we set the minimum number of SH candidates to $m=2$

\subsection{Program with Normal Distributions}

For SDVI, we use $10^3$ samples from the prior to discover SLPs.
To train the local guides to place support within the SLP boundaries we collect $10^2$ samples per SLP and optimize the objective in Equation~\eqref{eq:prior_density} for $10^3$ iterations. 
We run Algorithm~\ref{alg:vi_res_alloc} with a total budget of $T = 10^5$ with $5$ particles for the ELBO and to estimate the final SLP weights we use $10^3$ samples per SLP.
We use a learning rate of $0.01$.

For Pyro AutoGuide, we run the optimization for $10^5$ steps with $1$ ELBO particle. 
For BBVI, we run the optimization for $10^4$ steps with $10$ ELBO particles.
For both we use a learning rate of $0.01$.

\subsection{Infinite Gaussian Mixture Model}

For SDVI, we use $10^3$ samples from the prior to discover SLPs, run Algorithm~\ref{alg:vi_res_alloc} with a total budget of $T = 2 * 10^4$ with $10$ particles for the ELBO and to estimate the final SLP weights we use $10^2$ samples per SLP.
We use a learning rate of $0.1$.

For BBVI, we run for $2 * 10^4$ iterations using $10$ particles for the ELBO and a learning rate of $0.1$.
In the guide, we use a categorical distribution for number of components $K$ over the range $K \in [1, 25]$.
We ran initial experiments with instead using a Poisson distribution paramterized by the rate but we found this leads to an explosion in the number of components in the guide which resulted in the program running out of memory. 
For each $\mu_k$ the variational approximation is a diagonal Normal distribution parameterized by the mean and the diagonal entries in the covariance matrix. 

For DCC, we run for $200$ iterations, at each iteration we run $10$ independent RMH chains generating $10$ samples and to get a marginal likelihood estimate we use PI-MAIS \citepsupp{martino2017Layered} which places a proposal distribution (in our case a Gaussian) on the outputs of the RMH chains and samples from this proposal $M$ times; we set $M=10$.

\subsection{Inferring Gaussian Process Kernels}

\subsubsection{Model Details}

Our probabilistic context-free grammar for the kernel structure has the production rules
\begin{equation}
    \mathcal{K} \rightarrow \  \text{SE} \  | \  \text{RQ} \  | \  \text{PER} \  | \  \text{LIN} \  | \  \mathcal{K} \times \mathcal{K} \  | \  \mathcal{K} + \mathcal{K}.
\end{equation}
with the production probabilities $[0.2, 0.2, 0.2, 0.2, 0.1, 0.1]$. 
On each base kernel hyperparameter we place an $\text{InverseGamma}(\alpha = 2, \beta = 1)$ prior.
For each base kernel the specific hyperparameters we wish to do inference over are:\footnote{We use the same naming conventions as the Pyro Docs at \url{https://docs.pyro.ai/en/stable/contrib.gp.html##module-pyro.contrib.gp.kernels}.}
\begin{itemize}
    \item Squared Exponential (SE): Lengthscale
    \item Rational Quadratic (RQ): Lengthscale, Scale Mixture
    \item Periodic (PER): Lengthscale, Period
    \item Linear (LIN): Bias
\end{itemize}
Assuming we have $N$ observations with inputs $\mathbf{x} \in \mathbb{R}^N$ and outputs $\mathbf{y} \in \mathbb{R}^N$ our model can then be written as
\begin{align}
    \mathcal{K} \sim \text{PCFG}(), \quad  %
    \sigma \sim \text{HalfNormal}(0, 1), \quad %
    \mathbf{y} \sim \mathcal{N}(0, \mathcal{K}(\mathbf{x}) + \sigma^2 \text{I})
\end{align}
where $\text{PCFG}()$ samples a kernel (and its hyperparameters) from the probabilistic context-free grammar and $\mathcal{K}(\mathbf{x})$ is the $N \times N$ covariance matrix computed from kernel $\mathcal{K}$.

\subsubsection{Algorithm Configurations}

For SDVI, we use $10^3$ samples from the prior to discover SLPs, run Algorithm~\ref{alg:vi_res_alloc} with a total budget of $T = 10^6$ with $1$ particles for the ELBO and to estimate the final SLP weights we use $10^2$ samples per SLP.
We use a learning rate of $0.005$.

For BBVI, we run for $10^5$ iterations using $10$ particles for the ELBO and a learning rate of $0.005$.
The guide uses a log-normal distribution for the kernel hyperparameters and the observation noise, and for the discrete variables which influence the kernel structure we use categorical distributions.
For DCC, we run for $10^3$ iterations and otherwise use the exact same hyperparameters as in the Gaussian Mixture Model experiment.

\section{Additional Experimental Results}

\subsection{Program with Normal Distributions}

\begin{wrapfigure}{tr}{0.5\linewidth}
    \vspace{-20pt}
    \centering
    \includegraphics[width=\linewidth]{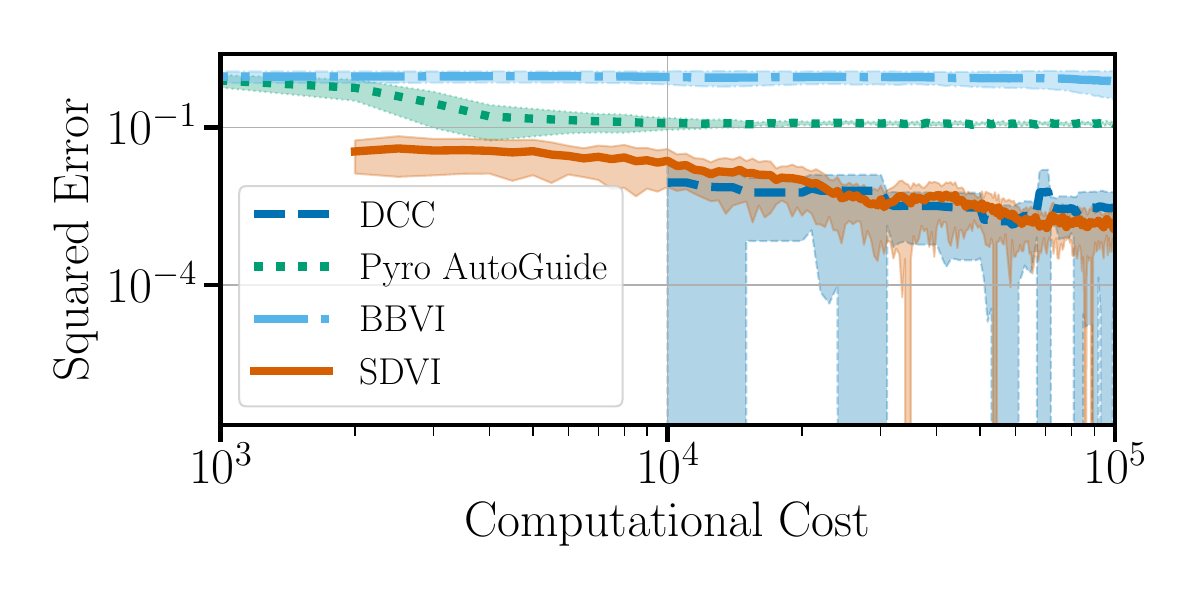}
    \vspace{-20pt}
    \caption{Squared error for the model in §~\ref{sec:numerical_example} with DCC baseline. Conventions as in Fig.~\ref{fig:numerical_example_weights_error}.}
    \label{fig:numerical_example_results_with_dcc}
    \vspace{-10pt}
\end{wrapfigure}
For completeness we include here the results for DCC on the model from Sec.~\ref{sec:numerical_example}.
DCC does not have the same fundamental limitations as the BBVI baselines therefore is competitive with SDVI and provides a similar squared error for the SLP weights.
In fact, it is quite impressive that SDVI is able to match the performance of DCC because DCC leverages marginal likelihood estimators which asymptotically converge to the true marginal likelihood whereas SDVI calculates the weights based on the ELBO.
This is therefore a further indicator that for this model SDVI is able to provide good posterior approximations for each SLP.

\section{Difficulties of Parameter Learning for Models with Stochastic Support}
\label{app:parameter}

In static support settings, one often uses variational bounds not only as a mechanism for inference, but also for training model parameters themselves~\citep{kingma2014AutoEncoding,rezende2014Stochastic}.
Using our notation from Sec.~\ref{sec:vi_background}, this setting corresponds to having model parameters, $\theta$, that we wish to optimize alongside the variational parameters, $\phi$, such that the unnormalized density can be written as $\gamma(x; \theta)$, with corresponding normalization constant $Z(\theta)$.
The ELBO then depends on both the variational and model parameters $\ELBO(\phi, \theta) := \mathbb{E}_{q(x; \phi)}\left[\log \nicefrac{\gamma(x; \theta)}{q(x; \phi)}\right]$.
Provided $Z(\theta)$ is differentiable with respect to $\theta$, both $\phi$ and $\theta$ can then, at least in principle, be simultaneously optimized using stochastic gradient ascent.

However, similar as to the case of pure inference, naively extending this scheme to models with stochastic support is non-trivial and quickly runs into both conceptual and practical problems.

Parameters, $\theta_k$, that are inherently local to only a single SLP can be dealt with straightforwardly: as $\nabla_{\theta_k} \ELBO_\ell = 0~ \forall \ell \neq k$ for such parameters, we can simply ignore parameters not associated with the SLP we are updating, that is we only take a gradient step for $\{\phi_k,\theta_k\}$ on Line~\ref{alg_line:slp_optimization} of Algo.~\ref{alg:vi_res_alloc}.

Problems start to occur, though, in the more common scenario where parameters are shared between SLPs, in the sense that they influence more than one $\gamma_k$.
Consider, for example, the GP model from Sec.~\ref{sec:gp_experiment} and assume that instead of doing inference over the observation noise, $\sigma$, we instead wish to treat this as a learnable parameter instead.
Here $\sigma$ could be seen as a `global' model parameter as it appears in every SLP, so could be viewed as shared between them.

This now creates an issue in `balancing' updates from different SLPs; the need to learn a shared $\theta$ breaks the separability between inference problems for individual SLPs.
Consequently, we can no longer directly treat how often we update each  SLPs as just a resource allocation problem: making more updates on a given SLP now increases the influence that SLP has on the $\theta$ which are learned.
This problem is unlikely to be insurmountable---one could maintain a running estimate of $q(k;\lambda)$ during training and then use this to either directly control the resource allocation or scale the updates of $\theta$ depending on how often the corresponding SLP has been used---but it does represent a notable complication that would require its own careful consideration. 

Beyond this specific practical challenge, there is also a more fundamental and general issue for parameter learning under stochastic support: should shared parameters be treated globally when we are learning them?
Going back to the example of the observation noise, $\sigma$, in our GP example, it will actually be quite inappropriate here to learn a single global value for $\sigma$, as the optimal observation noise will be different depending on the kernel structure.
Thus, though the variable is shared between SLPs in the program itself, it would be advantageous to learn separate values for it for each SLP, regardless of the inference approach we take.

The natural solution to this issue would be to perform parameter learning separately for each SLP, e.g. learning a separate $\sigma_k$ for each SLP in the GP example above.
However, this raises a variety of issues in its own, not least the fact that the inference algorithm will now start to influence the model itself: SDVI and BBVI will learn fundamentally different models.
There may also be settings where it is important for a parameter to be truly global and thus shared across the SLPs, e.g. because such sharing is an explicit prior assumption we wish to make.

Further problems occur when we consider that it is also feasible for learnable parameters to influence the control flow of the program, or even the set of possible SLPs.
For example, a learnable parameter could impact the maximum possible recursion depth of a recursive program.
This will create challenging interactions between SLPs: updates of one will influence the desirable behavior for the variational approximation of another.
In turn, this can substantially complicate the resource allocation process and even the SLP discovery process itself.

Together, these aforementioned issues demonstrate that parameter learning for models with stochastic support is a complex issue, requiring specialist consideration beyond the scope of the current paper.
\vspace{-5pt}

\section{Issues with Directly Training $\qslp$}
\label{app:direct_training}

A natural question one might ask with the SDVI method is why do we not directly train $\qslp$ to~\eqref{eq:global_kl_in_terms_of_elbos} by treating it as an implicit variational approximation defined by $\tildeqslp$? Namely, we can express~\eqref{eq:global_kl_in_terms_of_elbos} in terms of $\tilde{q}_k$ as follows
\begin{align}
 \mathcal{L}_k(\phi_k) = \log \tilde{Z}_k(\phi_k) + \frac{1}{\tilde{Z}_k(\phi_k)} \mathbb{E}_{\tilde{q}_k(x;\phi_k)}\left[ \mathbb{I}\left[x\in\mathcal{X}_k\right] \log \frac{\gamma_k(x)}{\tilde{q}_k(x;\phi_k)} \right],   
\end{align}
which, in principle, could be directly optimized with respect to $\phi_k$.

There are unfortunately two reasons that make this impractical.
Firstly, though $\tilde{Z}_k(\phi_k)$ can easily be estimated using Monte Carlo, we actually cannot generate conventional unbiased estimates of $\log \tilde{Z}_k(\phi_k)$ and $1/\tilde{Z}_k(\phi_k)$ (or their gradients) because mapping the Monte Carlo estimator induces a bias.  Second, this objective applies no pressure to learn a $\tilde{q}_k$ with a high acceptance rate, i.e. which actually concentrates on SLP $k$, such that it can easily learn a variational approximation that is very difficult to draw truncated samples from at test time.  

By contrast, using our surrogate objective in~\eqref{eq:surr_elbo} allows us to produce unbiased gradient estimates. Because of the mode seeking behaviour of variational inference, it also naturally forces us to learn a variational approximation with a high acceptance rate, provided we use a suitably low value of $c$.  If desired, one can even take $c \rightarrow 0$ during training to learn an approximation which only produces samples from the target SLP without requiring any rejection.  Figure~\ref{fig:acceptance_rates} shows that empirically we learn a $\tildeqslp$ with a very high acceptance rates for the problem in Section 6.1.

Note that the surrogate and true ELBOs are exactly equal for any variational approximation that is confined to the SLP (as these have $\tilde{Z}_k(\phi_k)=1$).  This does not always necessarily mean that they have the same optima in $\phi_k$ for restricted variational families, even in the limit $c\rightarrow 0$.  However, such differences originate from the fact that the trunctation can itself actually generalize the variational family (e.g. if $\tilde{q}_k$ is Gaussian, then $q_k$ will be a truncated Gaussians).  As such, any hypothetical gains from targeting~\eqref{eq:global_kl_in_terms_of_elbos} directly will always be offset against drops in the acceptance rate of the rejection sampler.

\begin{figure}[t]
  \centering
  \includegraphics[width=0.6\textwidth]{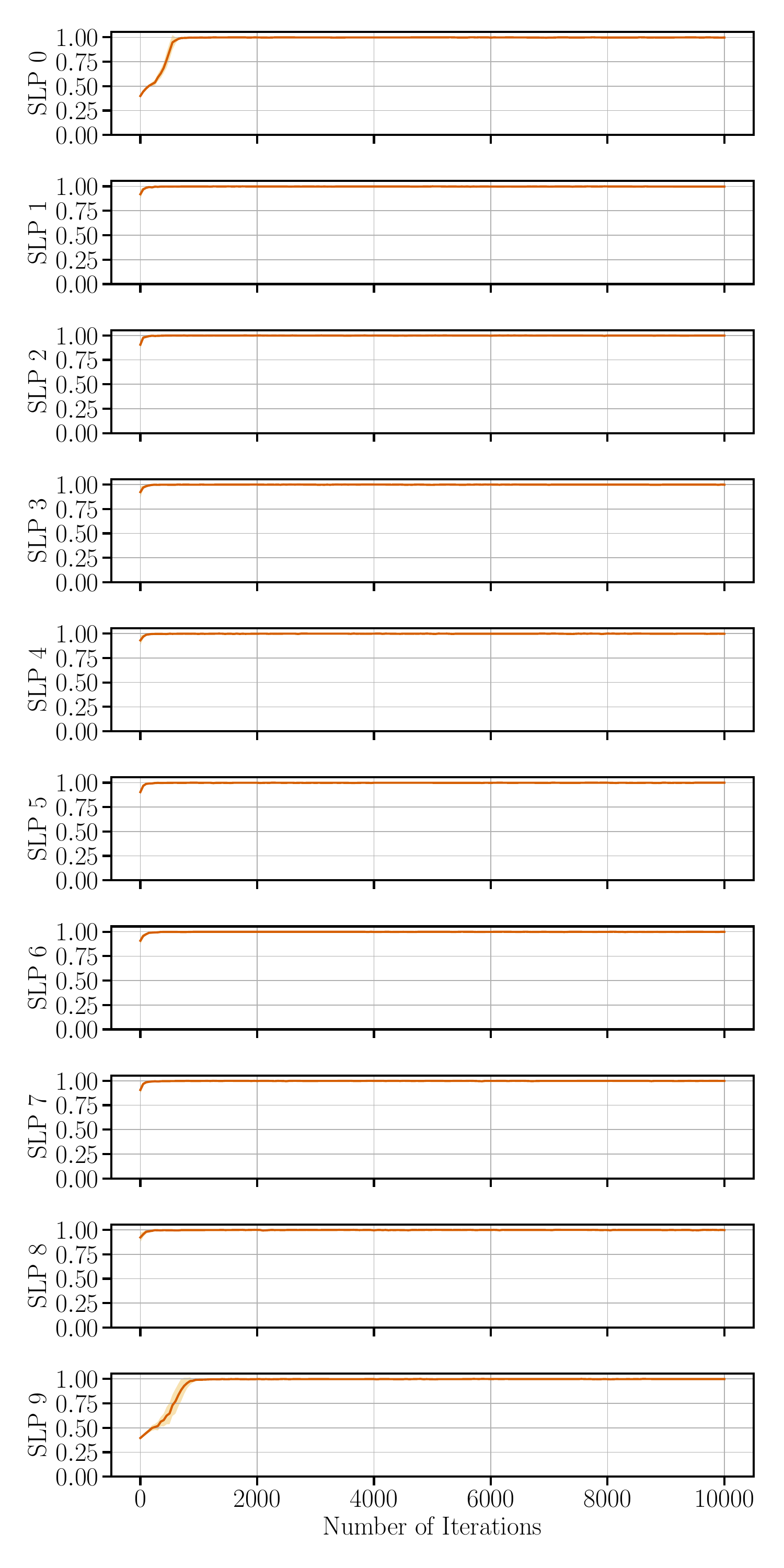}
  \caption{Acceptance rates for evaluating the local ELBOs in each SLP for the model from Sec.~\ref{sec:numerical_example}. Each plot represents a separate SLP; the plot with ``SLP {i}'' corresponds to the SLP with $z = \text{i}$ in Eq.~\eqref{eq:normal_model}. We can see that for all SLPs the acceptance rate approaches 1 with more iterations, confirming the mode seeking behaviour that arises when maximizing the surrogate ELBO in Eq.~\eqref{eq:surr_elbo}.}
  \label{fig:acceptance_rates}
\end{figure}

\bibliographysupp{references}
\bibliographystylesupp{unsrtnat}

\end{document}